\pdfoutput=1

\documentclass[11pt,table]{article}

\usepackage[final]{acl}

\usepackage{times}
\usepackage{latexsym}
\usepackage[T1]{fontenc}

\usepackage[utf8]{inputenc}

\usepackage{microtype}

\usepackage{inconsolata}

\usepackage{graphicx}
\usepackage{subcaption}

\usepackage{listings}
\newcounter{myverb}
\usepackage{hyperref}

\usepackage{enumerate}
\usepackage{mathtools}
\usepackage{amsmath}
\usepackage{amssymb}
\usepackage{amsthm}
\usepackage{booktabs}
\newenvironment{customlist}{
  \begin{list}{\textbullet}{
    \setlength{\leftmargin}{1em}
    \setlength{\itemindent}{0em}
    \setlength{\itemsep}{-0.1em}
  }
}{
  \end{list}
}

\makeatletter
\newtheorem*{rep@theorem}{\rep@title}
\newcommand{\newreptheorem}[2]{%
\newenvironment{rep#1}[1]{%
 \def\rep@title{#2 \ref{##1}}%
 \begin{rep@theorem}}%
 {\end{rep@theorem}}}
\makeatother

\newtheorem{proposition}{Proposition}
\newreptheorem{proposition}{Proposition}
\theoremstyle{remark}
\newtheorem*{remark}{Remark}
\renewcommand{\d}[1]{\ensuremath{\operatorname{d}\!{#1}}}

\usepackage{pifont}

\usepackage{xcolor}
\usepackage{multirow}
\usepackage{adjustbox}
\usepackage{longtable}
\usepackage{xtab}
\usepackage{adjustbox}
\usepackage{arydshln}
\usepackage{soul}
\sethlcolor{gray!15}
\usepackage[most]{tcolorbox}
\usepackage{verbatim}
\usepackage{listings}
\usepackage[frozencache,cachedir=.]{minted}
\usepackage{fancyvrb}
\fvset{fontfamily=\rmdefault}

\NewDocumentCommand{\VarField}{m}{\texttt{\{#1\}}}
\NewDocumentCommand{\Magenta}{m}{\textcolor{magenta}{#1}}

\DefineVerbatimEnvironment{PromptVerbatim}
  {Verbatim}
  {commandchars=\\\{\}, breaklines=true}

\usepackage{wrapfig}
\usepackage{enumitem}

\lstnewenvironment{myverbatim}[1][]{%
  \lstset{
    basicstyle=\ttfamily,
    frame=tb,
    #1
  }%
}{}

\BeforeBeginEnvironment{myverbatim}{%
\refstepcounter{myverb}%
}

\newcommand{\nocomment}{}

\ifx\nocomment\undefined
\NewDocumentCommand{\ec}
{ mO{} }{\textcolor{cyan}{\textsuperscript{\textit{Eunsol}}\textsf{\textbf{\small[#1]}}}}
\NewDocumentCommand{\mz}
{ mO{} }{\textcolor{blue}{\textsuperscript{\textit{Michael}}\textsf{\textbf{\small[#1]}}}}
\NewDocumentCommand{\vw}
{ mO{} }{\textcolor{gray}{\textsuperscript{\textit{Victor}}\textsf{\textbf{\small[#1]}}}}

\else
    \NewDocumentCommand{\ec}
    { mO{} }{\textcolor{cyan}{}}
      \NewDocumentCommand{\mz}
    { mO{} }{\textcolor{blue}{}}
      \NewDocumentCommand{\vw}
    { mO{} }{\textcolor{green}{}}
\fi

\title{Improving LLM-as-a-Judge Inference with the Judgment Distribution}

\author{
 \textbf{Victor Wang\textsuperscript{1}},
 \textbf{Michael J.Q. Zhang\textsuperscript{2}},
 \textbf{Eunsol Choi\textsuperscript{2}}
\\
Department of Computer Science \\
 \textsuperscript{1}The University of Texas at Austin,
 \textsuperscript{2}New York University
\\
 \texttt{victorwang37@utexas.edu}
}

\begin{document}
\maketitle
\begin{abstract}
Using language models to scalably approximate human preferences on text quality (LLM-as-a-judge) has become a standard practice applicable to many tasks.
A judgment is often extracted from the judge's textual output alone, typically with greedy decoding. However, LLM judges naturally provide \textit{distributions} over judgment tokens, inviting a breadth of inference methods for extracting fine-grained preferences.
We find that taking the mean of the judgment distribution consistently outperforms taking the mode (i.e. greedy decoding) in all evaluation settings (i.e. pointwise, pairwise, and listwise).
We further explore novel methods of deriving preferences from judgment distributions, and find that methods incorporating risk aversion often improve performance. Lastly, we analyze LLM-as-a-judge paired with chain-of-thought (CoT) prompting, showing that CoT can collapse the spread of the judgment distribution, often harming performance. Our findings show that leveraging distributional output improves LLM-as-a-judge, as opposed to using the text interface alone.
\end{abstract}

\section{Introduction}

\begin{figure*}
    \centering
    \includegraphics[width=\linewidth]{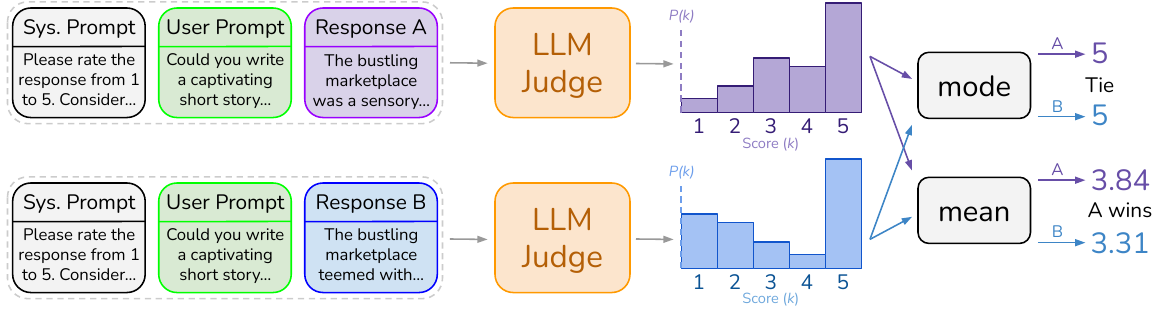}
    \caption{Pointwise LLM judge's logits produce a score distribution. We show two ways to compare two score distributions: (1) comparing the modes of the distributions and (2) comparing the means of the distributions.
    }
    \label{fig:pointwise}
\end{figure*}

LLM-as-a-judge has emerged as a scalable framework for evaluating model outputs by approximating human annotation \cite{lin2024wildbench, li2024crowdsourced, dubois2024length}.
Typically, such systems prompt off-the-shelf LLMs to score a response or rank multiple responses to a given user prompt.
LLM-as-a-judge methods boast strong agreement with human judgments across a breadth of domains and criteria \cite{zheng2023judging, ye2023flask}, despite current limitations \cite{koo2023benchmarking, tan2024judgebench}.



Most prior work involving LLM-as-a-judge elicits judgments through the LLM's text interface \cite{lin2024wildbench, zhu2023judgelm, ye2023flask}, where the most likely token (i.e. the mode of the next token distribution) or a sampled token is taken to represent the LLM's judgment. Recent works~\cite{leerlaif, liu2023g, Yasunaga2024ALMAAW} have suggested that taking the mean of the score token distribution can better represent the LLM's judgment.
In this work, we comprehensively evaluate design choices for leveraging LLM judges' distributional output.\footnote{We provide implementations of the evaluated methods at
\url{https://github.com/dubai03nsr/distributional-judge}.}

We show that the mean consistently outperforms the mode in the pointwise, pairwise, and listwise settings (i.e. evaluating one, two, and many responses at a time). Specifically, the mean achieves higher accuracy in 42 out of 48 cases on RewardBench \cite{lambert2024rewardbench} and MT-Bench \cite{zheng2023judging}.
We further explore novel methods of deriving preferences from score distributions (Section \ref{sec:study-score-distribution}).
For example, incorporating risk aversion often improves performance.
Categorizing methods as discrete or continuous, where discrete methods (e.g. mode) are simple to interpret like rubric scores,
we find that continuous methods outperform discrete methods, due to the latter often predicting ties and failing to capture slight preferences.
In particular, the mode assigns ties more frequently than every other method, leading to the lowest accuracy even among discrete methods.

We further study how chain-of-thought (CoT) prompting \cite{Wei2022ChainOT} impacts the performance of LLM-as-a-judge. After the CoT reasoning, LLMs often exhibit sharper score distributions, making the mean judgment similar to the mode. Removing CoT increases the spread of the judgment distribution, often improving performance, and more so for taking the mean than taking the mode (e.g. absolute +6.5\% for mean vs. +1.4\% for mode, on average with pointwise scoring on RewardBench), demonstrating the synergy between eliciting and using distributional output.


Our findings stress the importance of leveraging distributional output to maximize the effectiveness of LLM-as-a-judge, as opposed to using the text interface alone. As LLM-as-a-judge paradigms are widely adopted for complex tasks, improving best practices for using LLM-as-a-judge can impact many end tasks' development and evaluation.

\section{Background}

\subsection{LLM-as-a-Judge Settings}\label{sec:settings}

We briefly review three settings for LLM-as-a-judge; see Appendix \ref{app:rw-settings} for more background.

\paragraph{Pointwise Scoring} The LLM judge scores the two texts independently on a scale from 1 to some $K$, as shown in Figure \ref{fig:pointwise} \cite{zheng2023judging, lin2024wildbench, Cui2023UltraFeedbackBL}.

\paragraph{Pairwise Scoring} The LLM judge scores both texts in a single prompt \cite{zhu2023judgelm, Saha2023BranchSolveMergeIL, Chan2023ChatEvalTB}. To account for position bias, we prompt the LLM judge twice, once for each order of presentation, and average the outputs \cite{leerlaif}.

\paragraph{Pairwise Ranking} The LLM judge states which of the two texts it prefers \cite{lin2024wildbench, li2024crowdsourced, dubois2024length}. As with pairwise scoring, we prompt the LLM judge twice, once for each order of presentation.


\subsection{Related Work}
\paragraph{Mean Judgment}
Several prior works have used the mean of the judgment distribution, mostly in the pointwise setting. \citet{liu2023g, leerlaif, SaadFalcon2024LMUnitFE} note the benefits of the mean but do not empirically compare it with the mode.
\citet{zawistowski2024unused}, \citet{Hashemi2024LLMRubricAM}, \citet{Lukasik2024RegressionAI} show that the mean outperforms the mode for summary scoring, dialogue scoring, and other regression tasks.
Concurrent work~\citep{Yasunaga2024ALMAAW} shows that the mean outperforms the mode on RewardBench \cite{lambert2024rewardbench}, but the paper's focus is on data-efficient alignment.

\citet{leerlaif, Zhai2024OnlineSL} use pairwise judgment distributions to train a student model, but do not empirically compare with distillation using one-hot judgments. In this work, we benchmark the mode, the mean, and newly proposed methods for leveraging distributional judgments across the pointwise, pairwise, and listwise settings.

\paragraph{CoT}
\citet{zheng2023judging} presented preliminary evidence that CoT benefits LLM-as-a-judge. Other LLM-as-a-judge systems have been proposed that take advantage of LLMs' ability to perform CoT reasoning \cite{ankner2024critique, Feng2024NaturalLR}. On the other hand, \citet{Liu2024ReIFERI} evaluate many evaluation protocols and find that CoT can hurt performance. However, their analysis assumes access only to the judges' text interface, not examining the effect of CoT on the judgment distribution. In this work, we analyze the interplay between CoT and the inference method (e.g. mode vs. mean).

Related phenomena on the effect of CoT have been studied in the literature \cite{Chiang2023ACL, Stureborg2024LargeLM, Liu2024MindYS, Lee2023ApplyingLL, Sprague2024ToCO, hao2024traininglargelanguagemodels, zheng2023judging}. \citet{Wang2024ChainofThoughtRW} show the sharpening effect of CoT, which improves performance on numerical reasoning tasks. In this work, we show that this sharpening effect can be harmful when the LLM is used as a judge. 



\paragraph{Distributional Reward Models}
Using distributional judgment makes it possible for LLM judges to represent pluralistically aligned preferences \cite{sorensenposition, siththaranjan2023distributional, Kumar2024ComPOCP}. Compared to existing work on distributional reward models \cite{siththaranjan2023distributional, zhang2024diverging, Li2024AligningCF, Dorka2024QuantileRF, Poddar2024PersonalizingRL, Padmakumar2024BeyondTB}, (1) our setting involves LLMs not trained or prompted for distributional judgment~\cite{Meister2024BenchmarkingDA}, and (2) LLM judges can produce arbitrary distributions over a flexibly chosen discrete judgment space.

\section{Distributional Judgment}
\label{sec:pointwise-pairwise}
In this section, we present our findings comparing mode vs. mean inference and CoT vs. no-CoT prompting for LLM-as-a-judge systems.



\subsection{Methods}


To infer a judgment from the LLM's output distribution, we use the mode or the mean. With \textbf{mode}, we perform greedy decoding to produce a judgment token and discard the logits. With \textbf{mean}, we compute a weighted average of the judgment options, weighting each judgment option by the probability assigned to its token. See Appendix \ref{app:methods} for details.


\subsection{Experimental Setup}\label{sec:pointwise-pairwise-expmt}

\paragraph{Models} As LLM judges, we use gpt-4o-2024-08-06 (shortened to GPT-4o) \cite{openai2024gpt4ocard}, 
Llama-3.1-8B-Instruct (Llama-3.1-8B) \cite{dubey2024llama3herdmodels}, 
Mistral-7B-Instruct-v0.3 (Mistral-7B) \cite{jiang2023mistral7b}, and Prometheus-2-7B \cite{kim2024prometheus}. We cover a commonly used closed-source LLM\footnote{Some proprietary LLMs such as Anthropic Claude do not provide logit access, preventing us from including them in our experiments. In Appendix \ref{sec:deepseek-results}, we provide partial results for DeepSeek-V3, whose trends match those of GPT-4o.} (GPT-4o), as well as smaller open-source variants. 

\paragraph{Inference Settings}
We prompt the LLM judge with or without \textbf{CoT} reasoning, i.e. to provide a brief explanation before stating the judgment.
We use greedy decoding for CoT prompting.
See Appendix \ref{app:prompts} for prompts.

We softmax the judgment logits into judgment probabilities with temperature 1.
We use the score space $\{1, \ldots, K = 9\}$ in this section.

\begin{table}[t]
\small
\centering
\begin{adjustbox}{center=\columnwidth}
\begin{tabular}{ccccc}
\toprule
\multirow{2}{*}{Model} & \multirow{2}{*}{Setting} & \multirow{2}{*}{Method} & Reward & \multirow{2}{*}{MT-Bench} \\
&&& Bench &\\
\midrule
\multirow{6}{*}{\rotatebox{90}{\centering GPT-4o}} & \multirow{2}{*}{point score} &  mode & 85.1, 84.0 & 81.9, 80.5 \\
 &  &  mean & \underline{87.4}, \textbf{88.0} & \textbf{83.6}, \underline{83.2} \\ \cline{2-5}
 & \multirow{2}{*}{pair score} & mode & 86.7, \underline{87.4} & \underline{86.2}, \underline{86.5} \\
 &  & mean & \underline{87.1}, \textbf{87.6} & \underline{86.3}, \textbf{86.8} \\ \cline{2-5}
 & \multirow{2}{*}{pair rank} & mode & 88.4, 89.7 & 86.3, 85.6 \\
 &  & mean & 88.6, \textbf{90.5} & \textbf{87.3}, 85.9 \\ \hline
\multirow{6}{*}{\rotatebox{90}{\centering \shortstack{Llama\\-3.1-8B}}} & \multirow{2}{*}{point score} &  mode & 69.6, 72.2 & 74.9, 71.9 \\
 &  &  mean & 72.7, \textbf{79.3} & 78.7, \textbf{81.5} \\ \cline{2-5}
 & \multirow{2}{*}{pair score} & mode & 71.7, 75.2 & \textbf{82.6}, \underline{82.4} \\
 &  & mean & 72.1, \textbf{76.8} & \underline{82.3}, 81.2 \\ \cline{2-5}
 & \multirow{2}{*}{pair rank} & mode & 68.9, 58.9 & 76.2, 63.0 \\
 &  & mean & \textbf{74.2}, 68.6 & \textbf{80.0}, 76.5 \\ \hline
\multirow{6}{*}{\rotatebox{90}{Mistral-7B}} & \multirow{2}{*}{point score} &  mode & 60.4, 62.7 & 59.5, 66.2 \\
 &  &  mean & 63.8, \textbf{72.1} & 62.6, \textbf{74.0} \\ \cline{2-5}
 & \multirow{2}{*}{pair score} & mode & 67.3, 68.9 & \underline{79.3}, \underline{79.8} \\
 &  & mean & 68.1, \textbf{71.0} & \underline{80.0}, \textbf{80.4} \\ \cline{2-5}
 & \multirow{2}{*}{pair rank} & mode & 56.3, 53.8 & 51.5, 51.5 \\
 &  & mean & \textbf{63.9}, 59.1 & \textbf{73.5}, 65.5 \\ \hline
\multirow{6}{*}{\rotatebox{90}{\shortstack{Prometheus\\-2-7B}}} & \multirow{2}{*}{point score} &  mode & 64.3, 66.0 & 72.5, 73.5 \\
 &  &  mean & 64.6, \textbf{75.2} & 72.1, \textbf{81.6} \\ \cline{2-5}
 & \multirow{2}{*}{pair score} & mode & \textbf{71.0}, 68.7 & 78.4, \underline{80.8} \\
 &  & mean & 70.5, \underline{70.8} & 78.3, \textbf{80.9} \\ \cline{2-5}
 & \multirow{2}{*}{pair rank} & mode & 59.6, 48.2 & 51.5, 43.0 \\
 &  & mean & \textbf{69.7}, 48.8 & \textbf{75.4}, 33.4 \\
\bottomrule
\end{tabular}
\end{adjustbox}
\caption{Mode vs. mean and CoT vs. no-CoT (comma-separated) accuracy results (\%).
For each base model+setting, we bold the best result and underline results not significantly worse ($\alpha = 0.05$).
The mean outperforms the mode in 42 out of 48 cases. No-CoT outperforms CoT in 14 out of 16 cases when using the mean for pointwise or pairwise scoring.
}
\label{tab:pointwise-pairwise-short-main-results}
\end{table}

\paragraph{Evaluation Datasets and Metrics} We evaluate on RewardBench \cite{lambert2024rewardbench} and MT-Bench \cite{zheng2023judging}, two canonical datasets for preference modeling with human annotations. Each data instance contains a prompt, a preferred response, and a dispreferred response. 

We evaluate accuracy on the binary classification task; predicting the correct winner, a tie, or the wrong winner gets 1, 0.5, or 0 points, respectively \cite{lambert2024rewardbench}. RewardBench contains 2,985 (prompt, response 1, response 2) triplets, each labeled with the preferred response. Since MT-Bench has multiple human judgments per triplet, we compute accuracy using only triplets with unanimous human judgments (1,132 out of 1,814). See Appendix \ref{app:datasets} for dataset details.



\subsection{Results}\label{sec:pointwise-pairwise-main-results}

Table \ref{tab:pointwise-pairwise-short-main-results} shows our main results, comparing mode vs. mean and CoT vs. no-CoT across various prompt settings and LLMs.

\paragraph{Mean outperforms mode}

The mean outperforms the mode in 42 out of 48 cases.
In Table \ref{tab:pointwise-pairwise-main-results}, we provide a subset breakdown of RewardBench and observe particularly large gains for pointwise scoring on the Reasoning subset.

\begin{table}[t]
\small
\centering
\begin{tabular}{cccc}
\toprule
Model & Setting & RewardBench & MT-Bench \\
\midrule
\multirow{3}{*}{GPT-4o} & point score & .039, .103 & .041, .116 \\
 & pair score & .042, .066 & .038, .064 \\
 & pair rank & .002, .065 & .012, .114 \\ \midrule
\multirow{3}{*}{\shortstack{Llama\\-3.1-8B}} & point score & .060, .101 & .068, .093 \\
 & pair score & .054, .106 & .047, .092 \\
 & pair rank & .215, .318 & .186, .331 \\
\bottomrule
\end{tabular}
\caption{Average standard deviation of judgment distribution, with judgment options rescaled to $[0, 1]$. Comma-separated values in each cell are with and without CoT. No-CoT always has a greater standard deviation.
}
\label{tab:std}
\end{table}

\paragraph{CoT often harms LLM-as-a-judge}
For the scoring settings, no-CoT outperforms CoT in 14 out of 16 cases when using the mean. For the pairwise ranking setting, CoT outperforms no-CoT, except with GPT-4o on RewardBench.

We interpret the harmful effect of CoT on pointwise scoring with the smaller models as being due to \textit{sharpening}, whereby the initial entropy in the judgment is lost as the model commits to one instantiation of a reasoning trace.
Table \ref{tab:std} confirms this trend by showing that the standard deviation of judgment distributions is lower for CoT than no-CoT.
Moreover, removing CoT benefits the mean more than the mode (e.g. 69.6\textrightarrow72.2 for mode vs. 72.7\textrightarrow79.3 for mean, with Llama-3.1-8B on RewardBench), revealing the synergy between eliciting and utilizing distributional judgment.

\paragraph{Which setting works the best?}
Comparing different LLMs, we find GPT-4o performs better with pairwise judgment (e.g. 88.0 for pointwise scoring vs. 90.5 for pairwise ranking on RewardBench) as in prior work, but the smaller models often do better with pointwise judgment and rely heavily on CoT for pairwise ranking (e.g. with Prometheus-2-7B on MT-Bench, 75.4\textrightarrow33.4 when removing CoT from pairwise ranking, compared to 81.6 with no-CoT pointwise scoring).
We believe this is because pairwise judgment demands a more powerful judge to leverage the context. Thus, in pairwise ranking with the smaller models, the reasoning gained by CoT often outweighs the distributional signal lost in the process. Nonetheless, using pairwise scoring (where assigning individual scores can be viewed as an intermediate reasoning step) rather than pairwise ranking can eliminate the need for CoT, and we recover much of the gap on RewardBench, and match or exceed pointwise performance on MT-Bench.

\begin{table*}[t]
\small
\centering
\begin{tabular}{cllc}
\toprule
\multirow{2}{*}{Name} & \multirow{2}{*}{Description} & Definition of $\textsc{Name}(X_1, X_2) \in [-1, 1]$ & Discrete or \\
 & & (higher says $X_1$ is better, lower says $X_2$ is better) & Continuous\\\midrule
{\sc mode} & Mode & ${\rm sgn}(r_1 - r_2)$ with $r_i = \arg\max_k P(X_i = k)$ & Discrete\\
{\sc mean} & Mean & \hspace*{-2mm}\begin{tabular}{l}\(\displaystyle \frac{\mathbb{E}(X_1-X_2)}{\mathbb{E}|X_1-X_2| + \sigma(X_1 - X_2)} \)\end{tabular}\hspace*{2mm} & Continuous\\
{\sc[mean]} & Rounded mean & ${\rm sgn}(r_1 - r_2)$ with $r_i = \arg\min_k |\mathbb{E}X_i - k|$ & Discrete\\
{\sc medi} & Median & ${\rm sgn}(r_1 - r_2)$ with $r_i = Q_{X_i}(0.5)$ & Discrete\\
{\sc 1p} & 1st percentile & ${\rm sgn}(r_1 - r_2)$ with $r_i = Q_{X_i}(0.01)$ & Discrete\\
{\sc ram} & Risk-averse mean & $\textsc{mean}(X_1 - \sigma_-(X_1), X_2 - \sigma_-(X_2))$ & Continuous\\
{\sc qt} & Quantiles & $ \int_0^1 {\rm sgn}(Q_{X_1}(p) - Q_{X_2}(p)) \d{p}$ & Continuous\\
{\sc ps} & Probability of superiority & $P(X_1 > X_2) - P(X_1 < X_2)$ & Continuous\\
\bottomrule
\end{tabular}
\caption[]{Methods of comparing two score distributions $X_1, X_2$ over $K$ score options. $\rm sgn$ is the sign function. $Q_X(p)$ denotes the value at the $p$-quantile. $\sigma(X)$ denotes the standard deviation; $\sigma_-(X) = \sqrt{\mathbb{E}[\max(\mathbb{E}X - X, 0)^2]}$ denotes the lower semi-deviation, a risk measure \cite{Bond2002StatisticalPO}.
}
\label{tab:pointwise-methods}
\end{table*}

\begin{table}[t]
\small
\centering
\begin{tabular}{cccccc}
\toprule
\multirow{2}{*}{Model} & \multirow{2}{*}{Method} & \multicolumn{2}{c}{RewardBench} & \multicolumn{2}{c}{MT-Bench}\\\cmidrule(lr){3-4}\cmidrule(lr){5-6}
&& Acc $\uparrow$ & MSE $\downarrow$ & Acc $\uparrow$ & MSE $\downarrow$ \\
\midrule
\multirow{8}{*}{GPT-4o} & \sc mode & 84.0 & .118 & 80.5 & .145 \\
 & \sc mean & 88.0 & .102 & \underline{83.2} & \underline{.097} \\
 & \sc [mean] & 85.2 & .109 & 80.2 & .146 \\
 & \sc medi & 84.6 & .112 & 80.2 & .142 \\
 & \sc 1p & 84.3 & .116 & 81.0 & .138 \\
 & \sc ram & \textbf{88.4} & .100 & \textbf{83.4} & \textbf{.096} \\
 & \sc qt & 87.9 & \textbf{.096} & \underline{83.2} & .118 \\
 & \sc ps & 87.8 & \textbf{.096} & \underline{83.3} & .103 \\ \hline
\multirow{8}{*}{\shortstack{Llama\\-3.1-8B}} & \sc mode & 72.2 & .192 & 71.9 & .142 \\
 & \sc mean & 79.3 & .155 & \textbf{81.5} & .104 \\
 & \sc [mean] & 75.0 & .186 & 75.0 & .145 \\
 & \sc medi & 73.6 & .191 & 73.9 & .142 \\
 & \sc 1p & 76.0 & .183 & 79.2 & .147 \\
 & \sc ram & \textbf{79.9} & \textbf{.152} & \underline{81.4} & \textbf{.102} \\
 & \sc qt & 79.0 & .164 & 81.1 & .116 \\
 & \sc ps & 78.9 & .161 & \underline{81.4} & .110 \\
\bottomrule
\end{tabular}
\caption{Pointwise results over methods. No-CoT (see Table \ref{tab:pointwise-method-results} for CoT). Text styling follows Table \ref{tab:pointwise-pairwise-short-main-results}.
}
\label{tab:pointwise-method-nocot-results}
\end{table}

\section{Study on Pointwise Scoring}\label{sec:study-score-distribution}

Beyond the mode and mean discussed in prior work and the previous section, we further explore the design space of utilizing distributional output from LLM scorers.

\paragraph{Discrete vs. Continuous}
We say a method is \textit{discrete} if it compares two score distributions by their independently assigned scores that take values in $\{1, \ldots, K\}$. Otherwise, we say it is \textit{continuous}. Discrete scores are often desirable for interpretability (e.g. simple rubrics) but, by the pigeonhole principle, can often result in tied comparisons and fail to capture slight preferences.

\paragraph{Additional Metric: Mean Squared Error}

For our further analysis, we report mean squared error (MSE) in addition to accuracy.
For target labels in $\{0, 1\}$ (a unanimously preferred response), MSE is equivalent to the Brier score.
Accuracy incentivizes predicting a winner instead of a tie as long as oracle confidence is over 50\%.
In contrast, expected MSE is optimized by exactly predicting the oracle confidence, thus serving as a measure of a method's calibration given the judge's distributional output.

On MT-Bench, we generalize the label space to $[0, 1]$ by averaging the human judgments, thus allowing us to evaluate MSE on the full dataset.
In Appendix \ref{app:hetero-preferences}, we analyze alignment between the judgment distributions of LLMs and those of humans (as opposed to the average or majority vote).

\subsection{Methods}\label{sec:pointwise-methods}

Table \ref{tab:pointwise-methods} lists our extended methods for comparing two score distributions. We motivate the newly introduced methods below and provide details in Appendix \ref{app:pointwise-methods}.

Users often prefer discrete methods (e.g. mode) because they are simple to interpret, even if they have lower accuracy than continuous methods (e.g. mean). This motivates the question of where the mode (the status quo method) ranks among discrete methods. To answer this question, we compare the mode to other discrete methods: rounded mean, median, and first percentile (discussed in the next paragraph).

Humans exhibit risk aversion when making decisions. They often disprefer negative outcomes more strongly than they prefer positive outcomes \cite{holt2002riskaversion}. However, this disposition is not captured by the measures of central tendency discussed so far. Thus, we investigate whether incorporating the human disposition of risk aversion into LLM-as-a-judge inference methods improves alignment with human preferences. The methods {\sc 1p} (discrete) and {\sc ram} (continuous) reflect risk aversion. {\sc 1p} takes an approach contrary to {\sc mode}; instead of focusing on where the most mass lies, {\sc 1p} assigns a low score if there is even a 1\% chance of such a low score \cite{siththaranjan2023distributional}. {\sc ram} is {\sc mean} but with each distribution shifted down by its risk $\sigma_-$.

We have used {\sc mode} to represent the status quo LLM-as-a-judge inference method, which uses greedy decoding to obtain a judgment token. However, some prior works use a positive temperature, e.g. to obtain varied CoT chains \cite{zhang2024generative}, in which case a sampled judgment token is decoded rather than the mode. To account for the random nature of sampling, we design the method {\sc ps} as the difference in winrates over repeated pairs of samples from the LLM judge \cite{siththaranjan2023distributional}. {\sc qt} generalizes {\sc medi} and {\sc 1p} by averaging the comparisons over all quantiles, and can be viewed as {\sc ps} but with $X_1$ and $X_2$ positively monotonically correlated.

\subsection{Results}

\paragraph{Main Takeaways}
\begin{customlist}
    \item Table \ref{tab:pointwise-method-nocot-results} shows that the top pointwise methods are the continuous ones ({\sc mean}, {\sc ram}, {\sc qt}, {\sc ps}), in both accuracy and MSE, indicating that they should be chosen over discrete methods.

    \item Even among discrete methods, {\sc mode} has the lowest accuracy in 3 out of the 4 cases, indicating that the mode is a suboptimal choice even if discrete scores are desired.

    \item {\sc 1p} often outperforms {\sc medi} (e.g. 79.2 vs 73.9 accuracy with Llama-3.1-8B on MT-Bench), and {\sc ram} slightly outperforms {\sc mean} (e.g. 79.9 vs. 79.3 accuracy with Llama-3.1-8B on RewardBench), suggesting that risk aversion can be helpful for preference modeling.
\end{customlist}

\begin{table}[t]
\small
\centering
\setlength{\tabcolsep}{4pt}
\begin{adjustbox}{center=\columnwidth}
\begin{tabular}{cccccc}
\toprule
\multirow{2}{*}{Model} & \multirow{2}{*}{Method} & \multicolumn{2}{c}{Tie rate} & \multicolumn{2}{c}{{\sc mean}'s accuracy} \\\cmidrule(lr){3-4}\cmidrule(lr){5-6}
 & & $K=9$ & $K=99$ & $K=9$ & $K=99$ \\
\midrule
\multirow{4}{*}{GPT-4o} & \sc mode & .17 & .20 & 72 & 73 \\
 & \sc [mean] & .16 & .03 & 67 & 53 \\
 & \sc medi & .17 & .09 & 70 & 62 \\
 & \sc 1p & .16 & .08 & 66 & 60 \\ \hline
\multirow{4}{*}{\shortstack{Llama\\-3.1-8B}} & \sc mode & .35 & .24 & 69 & 70 \\
 & \sc [mean] & .26 & .07 & 64 & 61 \\
 & \sc medi & .29 & .11 & 67 & 67 \\
 & \sc 1p & .23 & .08 & 65 & 57 \\
\bottomrule
\end{tabular}
\end{adjustbox}
\caption{Tie analysis for discrete pointwise methods on RewardBench using no-CoT (see Table \ref{tab:pointwise-tie-rewardbench-results} for CoT and Table \ref{tab:pointwise-tie-mtbench-results} for MT-Bench). We report results with two score granularity levels ($K$). Tie rate is the proportion of instances where the method predicts a tie, over which we report {\sc mean}'s accuracy (\%); excess of 50\% or 75\% indicates room for improving accuracy or MSE, respectively.
}
\label{tab:pointwise-tie-rewardbench-small-results}
\end{table}

\begin{table}[t]
\small
\centering
\begin{tabular}{cccccc}
\toprule
& \multirow{2}{*}{Method} & \multicolumn{2}{c}{RewardBench} & \multicolumn{2}{c}{MT-Bench}\\\cmidrule(lr){3-4}\cmidrule(lr){5-6}
&& Acc $\uparrow$ & MSE $\downarrow$ & Acc $\uparrow$ & MSE $\downarrow$ \\
\midrule
\multirow{8}{*}{\rotatebox{90}{\centering GPT-4o}} & \sc mode & 81.7$_\text{--2.3}$ & .134$_\text{+.016}$ & 78.4$_\text{--2.1}$ & .158$_\text{+.013}$ \\
 & \sc mean & \textbf{86.7}$_\text{--1.3}$ & .108$_\text{+.006}$ & \underline{82.9}$_\text{--0.3}$ & .099$_\text{+.002}$ \\
 & \sc [mean] & \underline{86.5}$_\text{+1.3}$ & .127$_\text{+.018}$ & \underline{82.7}$_\text{+2.5}$ & .182$_\text{+.036}$ \\
 & \sc medi & 85.2$_\text{+0.6}$ & .126$_\text{+.014}$ & 81.5$_\text{+1.3}$ & .170$_\text{+.028}$ \\
 & \sc 1p & \underline{86.4}$_\text{+2.1}$ & .116$_\text{+.000}$ & \underline{82.7}$_\text{+1.7}$ & .165$_\text{+.027}$ \\
 & \sc ram & \textbf{86.7}$_\text{--1.7}$ & \textbf{.104}$_\text{+.004}$ & \textbf{83.0}$_\text{--0.4}$ & \textbf{.098}$_\text{+.002}$ \\
 & \sc qt & \underline{86.6}$_\text{--1.3}$ & .114$_\text{+.018}$ & \underline{82.7}$_\text{--0.5}$ & .147$_\text{+.029}$ \\
 & \sc ps & \underline{86.6}$_\text{--1.2}$ & \underline{.105}$_\text{+.009}$ & 82.4$_\text{--0.9}$ & .107$_\text{+.004}$ \\ \hline
\multirow{8}{*}{\rotatebox{90}{\centering Llama-3.1-8B}} & \sc mode & 72.0$_\text{--0.2}$ & .221$_\text{+.029}$ & 75.1$_\text{+3.2}$ & .169$_\text{+.027}$ \\
 & \sc mean & \underline{79.3}$_\text{+0.0}$ & .156$_\text{+.001}$ & \underline{81.3}$_\text{--0.2}$ & \underline{.103}$_\text{--.001}$ \\
 & \sc [mean] & 78.5$_\text{+3.5}$ & .198$_\text{+.012}$ & 80.7$_\text{+5.7}$ & .180$_\text{+.035}$ \\
 & \sc medi & 76.5$_\text{+2.9}$ & .207$_\text{+.016}$ & 80.1$_\text{+6.2}$ & .161$_\text{+.019}$ \\
 & \sc 1p & 78.5$_\text{+2.5}$ & .195$_\text{+.012}$ & \underline{81.5}$_\text{+2.3}$ & .177$_\text{+.030}$ \\
 & \sc ram & \textbf{79.7}$_\text{--0.2}$ & \textbf{.152}$_\text{+.000}$ & \underline{81.1}$_\text{--0.3}$ & \textbf{.102}$_\text{+.000}$ \\
 & \sc qt & 78.7$_\text{--0.3}$ & .177$_\text{+.013}$ & 81.3$_\text{+0.2}$ & .143$_\text{+.027}$ \\
 & \sc ps & 78.6$_\text{--0.3}$ & .163$_\text{+.002}$ & \textbf{81.8}$_\text{+0.4}$ & .111$_\text{+.001}$ \\
\bottomrule
\end{tabular}
\caption{Pointwise results over methods ($K=99$). No-CoT (see Table \ref{tab:pointwise-method-k99-results} for CoT). Subscripts denote change from $K=9$ (Table \ref{tab:pointwise-method-nocot-results}). Text styling follows Table \ref{tab:pointwise-pairwise-short-main-results}.}
\label{tab:pointwise-method-k99-nocot-results}
\end{table}

\begin{figure*}[t]
    \centering
    \includegraphics[width=\linewidth]{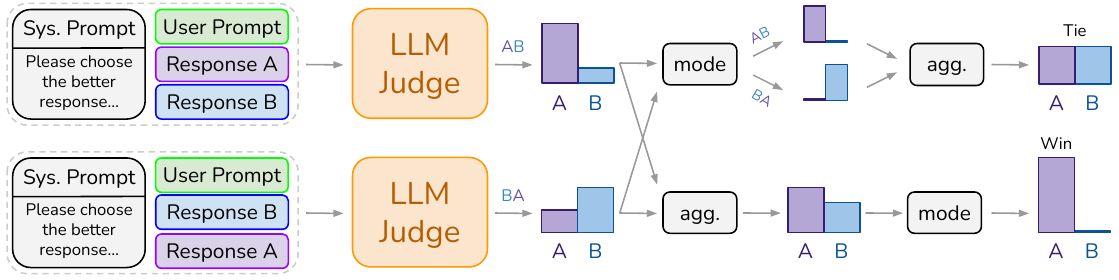}
    \caption{Comparing pairwise LLM-as-a-judge prediction based on \textbf{when} to aggregate the two judgments, one from each response pair presentation order. Pre- vs. post-aggregation (bottom vs. top in figure) can be likened to mean vs. mode, as the former aggregates at the distribution level while the latter aggregates at the text level (if mode is used).
    }
    \label{fig:pairwise}
\end{figure*}

\paragraph{Study: Score Granularity and Ties}
\label{sec:granularity-ties}

We show here that ties explain the finding above that the discrete methods fall behind the continuous ones, and we experiment with score granularity as a remedy.

Table \ref{tab:pointwise-tie-rewardbench-small-results} shows that the discrete methods predict ties on a significant number of instances, on which {\sc mean} is still able to achieve nontrivial accuracy. On the other hand, we find that on instances where a discrete method does not predict a tie, it has similar accuracy to {\sc mean} (not shown; see Table \ref{tab:pointwise-tie-rewardbench-results}), indicating that the performance gap is well explained by ties. Nonetheless, tie behavior varies by method; {\sc mode} has the most ties and the highest {\sc mean} accuracy, amounting to the most untapped signal for determining the better response.

Table \ref{tab:pointwise-tie-rewardbench-small-results} further shows that granularizing the score space from $K=9$ to $K=99$ improves the expressivity of the discrete methods (except for {\sc mode}), drastically reducing the rate of ties, while {\sc mean} accuracies remain similar or decrease.

Table \ref{tab:pointwise-method-k99-nocot-results} expands on the comparison between $K=99$ and $K=9$, reporting results from the same setting in Table \ref{tab:pointwise-method-nocot-results} except for the granularity scale.
Consistent with our motivation, the discrete methods (except for {\sc mode}) improve in accuracy, rivaling the continuous methods.
Although {\sc mode} somewhat makes up for its low accuracy with a lower MSE than most other discrete methods on MT-Bench, it suffers the highest MSE on RewardBench.

Taken together, Tables \ref{tab:pointwise-method-nocot-results}-\ref{tab:pointwise-method-k99-nocot-results}
show that even in use cases where discrete scores are desired, one should consider alternatives to the mode.

\paragraph{Sensitivity to Score Granularity}
In Appendix \ref{app:sensitivity-to-score-granularity}, we analyze the sensitivity of different methods to score granularity, and find theoretically and empirically that the mode is the most sensitive method.

\section{Study on Pairwise Ranking}\label{sec:study-pairwise-ranking}

The judgment styles in Section \ref{sec:pointwise-pairwise}'s overview were scoring (Section \ref{sec:study-score-distribution}) and ranking. In this section, we analyze design decisions for pairwise ranking, and in Section \ref{sec:listwise} listwise ranking.

\subsection{Design Decisions}

As we explain below, the pairwise ranking experiments in Table \ref{tab:pointwise-pairwise-short-main-results} used Likert-2, post-aggregation for the mode, and pre-aggregation for the mean. We now consider alternative choices (see Appendix \ref{app:pairwise-ranking} for details).

\paragraph{Timing of aggregation and measure of central tendency}
Pairwise judgment suffers from position bias, i.e. the LLM judge's sensitivity to the order in which the evaluated texts are presented, which is usually addressed by prompting the LLM judge twice, once for each order of presentation \cite{leerlaif}. We examine the remaining question of whether to aggregate the two judgments \textit{before or after} computing the measure of central tendency (mode, median, or mean), as shown in Figure \ref{fig:pairwise}. Pre- vs. post-aggregation can be likened to mean vs. mode, as the former aggregates at the distribution level while the latter aggregates at the text level (if mode is used).

\paragraph{Granularity} We prompt the judge to express its preference on a $K$-point Likert scale: $[>, <]$ (Likert-2), $[>, =, <]$ (Likert-3), or $[\gg, >, =, <, \ll]$ (Likert-5) \cite{Liu2024RewardMW}. In the prompt (Appendix \ref{app:pairwise-prompts}), the symbols are assigned descriptions such as ``significantly better'' and ``slightly better''.

\begin{table}[t]
\small
\centering
\setlength{\tabcolsep}{4pt}
\begin{adjustbox}{center=\columnwidth}
\begin{tabular}{cccccc}
\toprule
\multirow{2}{*}{Center} & \multirow{2}{*}{\shortstack{Agg.\\Time}} & \multicolumn{2}{c}{RewardBench} & \multicolumn{2}{c}{MT-Bench} \\\cmidrule(lr){3-4}\cmidrule(lr){5-6}
 & & Acc $\uparrow$ & MSE $\downarrow$ & Acc $\uparrow$ & MSE $\downarrow$ \\
\midrule
\multirow{2}{*}{mode} & post & 56.7 & \textbf{.240} & 57.5 & \textbf{.192} \\
 & pre & \textbf{73.1} & .265 & \textbf{78.1} & .222 \\ \hline
\multirow{2}{*}{median} & post & 56.8 & \textbf{.240} & 57.5 & \textbf{.192} \\
 & pre & \textbf{72.9} & .261 & \textbf{78.0} & .218 \\ \hline
\multirow{2}{*}{mean} & post & \underline{73.2} & \textbf{.207} & \textbf{78.2} & \textbf{.144} \\
 & pre & \textbf{73.2} & .222 & \underline{78.1} & .155 \\
\bottomrule
\end{tabular}
\end{adjustbox}
\caption{Pairwise ranking results over methods using Likert-3 comparing pre- and post-aggregation. All methods use Llama-3.1-8B, CoT (see Table \ref{tab:pairwise-method-results} for GPT-4o and no-CoT). Text styling follows Table \ref{tab:pointwise-pairwise-short-main-results}.}
\label{tab:pairwise-method-short-results}
\end{table}

\begin{table}[t]
\small
\centering
\begin{adjustbox}{center=\columnwidth}
\begin{tabular}{ccccc}
\toprule
\multirow{2}{*}{$K$} & \multicolumn{2}{c}{RewardBench} & \multicolumn{2}{c}{MT-Bench} \\\cmidrule(lr){2-3}\cmidrule(lr){4-5}
 & Acc $\uparrow$ & MSE $\downarrow$ & Acc $\uparrow$ & MSE $\downarrow$ \\
\midrule
2 & \textbf{74.2}, 68.6 & \textbf{.187}, .214 & \textbf{80.0}, 76.5 & \textbf{.126}, .135 \\
3 & \underline{73.2}, 66.3 & .222, .240 & 78.1, 70.8 & .155, .155 \\
5 & 70.0, 58.5 & .215, .234 & 77.1, 64.8 & .142, .153 \\
\bottomrule
\end{tabular}
\end{adjustbox}
\caption{Pairwise ranking results over Likert-$K$ scales, using pre-aggregation mean. Llama-3.1-8B, CoT (see Table \ref{tab:pairwise-vspace-results} for GPT-4o and no-CoT). Text styling follows Table \ref{tab:pointwise-pairwise-short-main-results}.}
\label{tab:pairwise-vspace-short-results}
\end{table}

\subsection{Methods Results}
Table \ref{tab:pairwise-method-short-results} shows that accuracy depends little on the measure of central tendency and mostly on when we aggregate, with aggregating first leading to higher accuracy (as much as 56.7\textrightarrow73.1 using the mode on RewardBench).
Considering that the timing of aggregation does not affect accuracy if the two runs agree, this shows that \textbf{even for inconsistent judgments caused by position bias, there is still valuable signal in the relative magnitudes of preference that we can leverage by aggregating first}.

On the other hand, an intuitive explanation for why the measure of central tendency has little effect on accuracy is that the judgment space is small, so there is high correlation between the signs of the measures of central tendency. In fact, they are equivalent in the pre-aggregation Likert-2 setting.

Although aggregating first improves accuracy, it harms MSE for mode and median, which we attribute to the volatile prediction of a binary winner when faced with the uncertain situation of positional inconsistency. Nevertheless, the mean (with either pre- or post-aggregation) is among the top accuracy methods while outperforming all other methods on MSE. This demonstrates the calibration benefit of using the judgment distribution to produce a continuous prediction.

With GPT-4o (not shown; see Tables \ref{tab:pairwise-method-results}, \ref{tab:pairwise-vspace-results}), MSE is always minimized with no-CoT, highlighting the discord between CoT's sharpening effect and calibration.
In Appendix \ref{app:position-bias}, we further analyze position bias and find that CoT increases the occurrence of severe position bias.

\subsection{Granularity Results}\label{sec:pairwise-vspaces}

Table \ref{tab:pairwise-vspace-short-results} compares the Likert scales used in the pairwise ranking prompt. We find that \textbf{Likert-2 performs the best overall},
in line with the AlpacaEval methodology \cite{dubois2024length} but deviating from WB-Reward and Arena-Hard-Auto \cite{lin2024wildbench, li2024crowdsourced}, which use Likert-5.


\section{Listwise Judgment}\label{sec:listwise}

Listwise judgment is not as prevalent as pointwise or pairwise judgment, but it offers efficiency~\cite{zhu2024starling} while granting the judge the maximal context for comparison \cite{Buyl2023RankFormerLL}.

\subsection{Judgment Spaces and Methods}\label{sec:listwise-methods}

We consider two prompts for eliciting listwise preferences over $N$ texts (Appendix \ref{app:listwise-prompts}). Prompt 1 is the one proposed by \citet{zhu2024starling}, which prompts to produce all $\binom{N}{2}$ pairwise preferences and then aggregate them into a sorted list. Prompt 2 skips the intermediate pairwise step and asks to directly produce the list \cite{Liu2023XEvalGM, qin2023large}. We can then extract all pairwise\footnote{We retain the pairwise evaluation setup from previous sections; see Appendix \ref{app:list-eval} for discussion.} preferences from one of the following judgment spaces using the mode (textual output) or the mean (distributional output).

\begin{customlist}
    \item {\sc interm} (Prompt 1): Intermediate pairwise preferences (Likert-3, no-CoT, only one of the two presentation orders), which we view as the reasoning process leading to the list. This efficiently extends pairwise ranking to the listwise setting, similar to batch prompting \cite{Cheng2023BatchPE}.

    \item {\sc list} (Prompt 1): Final list.
    For {\sc mean}, we use the probability distribution over text identifiers at each rank, inspired by \citet{Zhuang2023ASA, reddy2024first}.
    Specifically, at rank $r$, denote $p_r(i)$ as the probability of decoding text (identifier) $i$.
    Decoding text $i$ at rank $r$ implies that any text $j$ not yet decoded will be decoded at a later rank and is thus worse than text $i$, and vice versa.
    Hence, we define $\textsc{mean}(i, j) \in [0, 1]$ as the average of $\frac{p_r(i)}{p_r(i) + p_r(j)}$ over the ranks $r$ until $i$ or $j$ is decoded.

    \item {\sc direct list} (Prompt 2): {\sc list} but with Prompt 2 (no intermediate pairwise step).
\end{customlist}

\subsection{Experimental Setup}\label{sec:listwise-experimental-setup}

\paragraph{Models} Due to the context length required for listwise ranking and the difficulty of the task, we limit our evaluation to GPT-4o. In preliminary experiments, we found poor performance with the smaller models, but in Appendix \ref{sec:deepseek-results} we show that DeepSeek-V3 exhibits similar trends to GPT-4o.

\paragraph{Datasets} We evaluate on Nectar \cite{zhu2024starling}, RM-Bench \cite{Liu2024RMBenchBR}, and MT-Bench \cite{zheng2023judging}.

From Nectar, we use a random subset of 1,000 prompts, each with 7 responses. We discard the GPT-4 judgments included in the dataset and collect our own silver labels using GPT-4o with pairwise ranking (Likert-5, no-CoT, pre-aggregation, mean).
RM-Bench contains 1,327 prompts, each with 3 chosen and 3 rejected responses, yielding 9 pairwise preference labels.
MT-Bench contains 160 prompts, each with 6 responses.
See Appendix \ref{app:datasets} for dataset details.

\subsection{Results}

\begin{table}[t]
\small
\begin{adjustbox}{center=\columnwidth}
\centering
\setlength{\tabcolsep}{3.5pt}
\begin{tabular}{cccccccc}
\toprule
\multirow{2}{*}{Space} & \multirow{2}{*}{Method} & \multicolumn{2}{c}{Nectar} & \multicolumn{2}{c}{RM-Bench} & \multicolumn{2}{c}{MT-Bench} \\\cmidrule(lr){3-4}\cmidrule(lr){5-6}\cmidrule(lr){7-8}
 & & Acc & MSE & Acc & MSE & Acc & MSE \\
\midrule
\multirow{2}{*}{interm} & mode & \textbf{80.4} & .155 & 62.1 & .339 & \textbf{80.8} & .201 \\
 & mean & \textbf{80.4} & \textbf{.048} & \textbf{62.5} & \textbf{.243} & \underline{80.7} & \textbf{.121} \\ \hline
\multirow{2}{*}{list} & mode & \textbf{82.2} & .156 & \textbf{62.4} & .376 & \textbf{83.7} & .189 \\
 & mean & \underline{82.0} & \textbf{.105} & 61.7 & \textbf{.317} & \underline{83.5} & \textbf{.157} \\ \hline
\multirow{2}{*}{direct list} & mode & 86.1 & .138 & \textbf{69.9} & .301 & \textbf{86.8} & .168 \\
 & mean & \textbf{86.4} & \textbf{.087} & 69.4 & \textbf{.267} & 85.9 & \textbf{.133} \\
\bottomrule
\end{tabular}
\end{adjustbox}
\caption{Listwise results (GPT-4o). Text styling follows Table \ref{tab:pointwise-pairwise-short-main-results}.}
\label{tab:listwise-results}
\end{table}

Table \ref{tab:listwise-results} compares mode and mean in the listwise judgment spaces. The two methods have similar accuracy, but the mean has much lower MSE.

We find {\sc direct list} to be the most accurate judgment space (notably, outperforming pointwise scoring on MT-Bench; see Table \ref{tab:pointwise-method-nocot-results}), while {\sc interm} has the lowest MSE.
We hypothesize that {\sc direct list} outperforms {\sc list} due to the intermediate pairwise comparisons playing a similar role to CoT in the pointwise and pairwise settings, where distributional output is captured most intactly without it.
Even so, in Appendix \ref{app:position-bias} we find {\sc direct list} to suffer the most position bias, consistent with \citet{zhu2024starling}, while {\sc interm} has the least.

\section{Conclusion and Recommendations}

We comprehensively evaluated design choices for leveraging LLM judges' distributional output.
For pointwise scoring, we showed that continuous methods (e.g. mean) outperform discrete methods (especially the mode) due to ties.
For pairwise ranking, we related the mean vs. mode comparison to pre- vs. post-aggregation of the two presentation orders' judgments.
Although smaller LLM judges suffer heavily from inconsistent judgments due to position bias, pre-aggregation effectively leverages the relative magnitudes of preference.

We showed that CoT collapses the spread of the judgment distribution, often hurting performance. This applies even to the challenging setting of listwise ranking, where accuracy was maximized by directly predicting the list without an intermediate pairwise step. We hope that highlighting this limitation of CoT encourages the development of reasoning mechanisms that preserve output diversity and calibration for judgment and other subjective or open-ended tasks.

\paragraph{Recommendations}

We summarize our findings into guidelines for choosing judgment settings. Large judges like GPT-4o should use pairwise ranking no-CoT, or direct listwise ranking as an efficient alternative. Smaller judges like Llama-3.1-8B should use pointwise scoring no-CoT. The mean should be used instead of the mode, but these setting guidelines apply even if one uses the mode.

\section*{Limitations}

\paragraph{Downstream Performance}
In this paper, we evaluate LLM-as-a-judge design decisions by their performance on preference modeling datasets. However, this setup may not reveal downstream impacts. We do not explore the impact of distributional judgments on reinforcement learning from AI feedback (RLAIF) \cite{leerlaif} or human decision making.

\paragraph{Training}
Our experiments involve off-the-shelf LLMs as judges without specific tuning. We do not explore training LLM judges to express distributional judgments \cite{SaadFalcon2024LMUnitFE}. Similarly, we exclude distributional reward models \cite{Dorka2024QuantileRF} from the scope of our study.

\paragraph{CoT}
We conclude from our results that CoT often hurts judgment performance. However, we only consider one prompt design per setting for eliciting CoT reasoning (Appendix \ref{app:prompts}) and do not perform prompt optimization. Furthermore, we do not consider more extensive test-time scaling, such as asking the judge to produce its own reference response \cite{zheng2023judging} or aggregating many CoT judgment runs \cite{zhang2024generative, Stureborg2024LargeLM}.

\paragraph{Natural Language Judgments}
A valuable aspect of LLM-as-a-judge is its ability to augment judgments with interpretable rationales \cite{mahan2024generative, Byun2024ARESAR, Ye2024BeyondSR, Cao2024CompassJudger1AJ}. However, the distributional judgments we consider here are limited to those that are easily quantifiable, and we do not propose methods for leveraging distributional output over natural language feedback. While it is possible to continue decoding a rationale after the judgment, the rationale will be conditioned on the decoded judgment and not reflect the distribution over the unchosen judgment options. One approach could be to decode several rationales, each conditioned on a different judgment option.


\bibliography{main}

\appendix

\section{Related Work: LLM-as-a-judge Settings}\label{app:rw-settings}

LLM-as-a-judge has been used in pointwise (evaluating one response at a time), pairwise (two), and listwise (many) settings.

Pairwise judgment has the advantage of grounding each evaluated response in the other, creating for a more calibrated task and leading to better agreement with humans \cite{liusie2023zero}. However, due to intransitivity in pairwise preferences \cite{liu2024aligning}, the cost to sort $N$ texts is $O(N^2)$ rather than $O(N\log{N})$, compared to $O(N)$ in the pointwise setting. In addition, pairwise comparisons are susceptible to position bias \cite{shi2024judging}, which often must be addressed by running both orders and aggregating the results \cite{zeng2023evaluating, li2024crowdsourced}. Pairwise comparisons have also been shown to be more biased toward superficial traits such as verbosity and tone, in both LLM and human judges \cite{PAYNE1976366, jeong2024prepair}, although pointwise scoring more easily falls victim to adversarial responses \cite{raina2024llm}.

The listwise setting provides the maximal amount of context to the judge while keeping the same compute complexity as the pointwise setting. However, the judgment task becomes much more challenging \cite{qin2023large, koo2023benchmarking}, especially due to the amplified position bias \cite{zhu2024starling}, and the combinatorially many orders makes it severely more daunting to address than in the pairwise case \cite{Tang2023FoundIT, Qin2024LAMPOLL}. To mitigate position bias, \citet{zhu2024starling} leverage intermediate pairwise preferences for aggregation into a sorted list. \citet{Zhuang2023ASA, reddy2024first} use the distribution from a single output token for listwise passage reranking, a related task to LLM-as-a-judge.

\section{Methods}\label{app:methods}

Let $A_1$ and $A_2$ be two texts to compare.
We describe the methods of predicting a value in $[-1, 1]$ that signifies the advantage of $A_1$ over $A_2$.
For accuracy, we take the sign of the prediction. For MSE, we rescale predictions from $[-1, 1]$ to $[0, 1]$.

The prompts for the various settings are in Appendix \ref{app:prompts}.

\subsection{Pointwise Methods}\label{app:pointwise-methods}

We elaborate on the pointwise methods introduced in Section \ref{sec:pointwise-methods}. The LLM judge independently judges $A_1$ and $A_2$, producing score distributions over $\{1, \ldots, K\}$ for an integer $K$ that define independent random variables $X_1$ and $X_2$, which are used to compare $A_1$ and $A_2$.

The methods are invariant to scaling and translating the judgment space, and all methods that do not take expectations $\mathbb{E}$ (which assumes linearity) are invariant to applying a positive monotone transformation to the judgment space.
The methods are all equivalent if the distributions are deterministic, thus our experiments evaluate their ability to leverage the LLM judge's \textit{distributional} output.

The denominator in {\sc mean} normalizes it into $[-1, 1]$, similar to ${\rm sgn}(x) = \frac{x}{|x|}$, taking $\frac00$ to be 0. The $\sigma$ term lowers the magnitude of the prediction in the presence of uncertainty in a continuous manner. Specifically, let $k, k' \in \{1, \ldots, K\}$ with $k \neq k'$. For $\epsilon \in [0, 1]$, let $X_1$ have a two-point distribution $(1-\epsilon)\delta_k + \epsilon\delta_{k'}$ and let $X_2$ have a deterministic distribution $\delta_k$. Then $\textsc{mean}(X_1, X_2)$ as a function of $\epsilon$ is continuous at $\epsilon = 0$.

For {\sc mean}, {\sc ram}, and {\sc ps}, we assume $X_1$ and $X_2$ to be independent, but {\sc qt} can be viewed as {\sc ps} but with $X_1$ and $X_2$ positively monotonically correlated.
By incorporating the sign function, {\sc qt} and {\sc ps} are less sensitive to extremal values than {\sc mean}. 
In addition, {\sc qt} and {\sc ps} can model intransitive preferences, e.g. $\textsc{ps}(X_1, X_2),\ \textsc{ps}(X_2, X_3) > 0 \nRightarrow \textsc{ps}(X_1, X_3) > 0$, which we analyze in Appendix \ref{app:transitivity}.


\subsection{Pairwise Methods}\label{app:pairwise-methods}

In the pairwise setting, we consider two prompting approaches for jointly evaluating the two texts $A_1$ and $A_2$: scoring both texts (\S\ref{app:pairwise-scoring}) and expressing a preference (\S\ref{app:pairwise-ranking}).

To account for position bias, we prompt the LLM judge once for each order of presentation. For an order $\mathbf{o} \in \mathbf{O} \coloneq \{(1, 2), (2, 1)\}$, we use $\mathbf{o}$ to denote dependence on the order $(A_{o_1}, A_{o_2})$ in which the texts appear in the prompt.

\subsubsection{Pairwise Scoring}\label{app:pairwise-scoring}

For a given order $\mathbf{o}$, the LLM judge scores the two texts jointly in the same run. If we could obtain the joint distribution $P(X^\mathbf{o}_{o_1}, X^\mathbf{o}_{o_2})$, we could compute the marginals and use any method in Table \ref{tab:pointwise-methods}. However, the judge first outputs the score for $A_{o_1}$ and conditions on it when outputting the score for $A_{o_2}$, i.e. $X^\mathbf{o}_{o_1}$ and $X^\mathbf{o}_{o_2}$ are not independent. Thus, the full joint distribution $P(x_{o_1}, x_{o_2}) = P(x_{o_1}) P(x_{o_2} \mid x_{o_1})$ can only be obtained by injecting each $x_{o_1} \in \{1, \ldots, K\}$ into the context to access $P(x_{o_2} \mid x_{o_1})$. This is feasible with local models but not with API-access models where inference cost scales with $K$. Hence, we stick to a single run and condition on the greedily decoded $x_{o_1} = \arg\max_{k} P(X^\mathbf{o}_{o_1} = k)$, giving us
\begin{align*}
    X^\mathbf{o}_\Delta \overset{d}{=} (X^\mathbf{o}_1 - X^\mathbf{o}_2) \mid (X^\mathbf{o}_{o_1} = x_{o_1})
\end{align*}
as a proxy for the score difference $X^\mathbf{o}_1 - X^\mathbf{o}_2$. Semantically, $X^\mathbf{o}_\Delta$ is symmetric (i.e. there should be no prior preference for $A_1$ or $A_2$), so we would like our scalar judgment to be some measure of \textit{central} tendency (mode, median, or mean).
As shown in Figure \ref{fig:pairwise}, we also have the choice of whether to aggregate the judgments from the two orders of presentation \textit{before or after} computing the measure of central tendency.

For pre-aggregation, we simply take the mixture distribution,
\begin{align*}
    P(X_\Delta = \delta) \coloneq \frac{1}{|\mathbf{O}|}\sum_{\mathbf{o}\in\mathbf{O}} P(X^\mathbf{o}_\Delta = \delta)
\end{align*}
for all $\delta \in \{-(K-1), \ldots, K-1\}$,
leaving more sophisticated approaches such as the convolution and Wasserstein barycenter for future study:
\begin{align*}
    \textsc{agg-mode} &\coloneq {\rm sgn}({\rm mode}(X_\Delta))\\
    \textsc{agg-medi} &\coloneq {\rm sgn}({\rm median}(X_\Delta))\\
    \textsc{agg-mean} &\coloneq \textsc{mean}(X_\Delta),
\end{align*}
where {\sc mean} is defined as in Table \ref{tab:pointwise-methods}, overloaded to take a single argument representing $X_1-X_2$.

For post-aggregation, we sum the two scalar judgments from the two orders and normalize:
\begin{align*}
    \textsc{mode-agg} &\coloneq \frac{\sum_{\mathbf{o} \in \mathbf{O}}{\rm mode}(X^\mathbf{o}_\Delta)}{\sum_{\mathbf{o} \in \mathbf{O}}|{\rm mode}(X^\mathbf{o}_\Delta)|}\\
    \textsc{medi-agg} &\coloneq \frac{\sum_{\mathbf{o} \in \mathbf{O}}{\rm median}(X^\mathbf{o}_\Delta)}{\sum_{\mathbf{o} \in \mathbf{O}}|{\rm median}(X^\mathbf{o}_\Delta)|}\\
    \textsc{mean-agg} &\coloneq \frac{1}{|\mathbf{O}|}\sum_{\mathbf{o} \in \mathbf{O}} \textsc{mean}(X^\mathbf{o}_\Delta),
\end{align*}
taking $\frac00 \coloneq 0$.

\subsubsection{Pairwise Ranking}\label{app:pairwise-ranking}

We prompt the LLM judge to express its preference on a $K$-point Likert scale: $[>, <]$ (Likert-2), $[>, =, <]$ (Likert-3), or $[\gg, >, =, <, \ll]$ (Likert-5).
Assigning the symbols $[\gg, >, =, <, \ll]$ the numerical values $[2, 1, 0, -1, -2]$, the methods for pairwise ranking then follow those above for pairwise scoring. We remark that the `mode' and `mean' for pairwise scoring and pairwise ranking in Table \ref{tab:pointwise-pairwise-short-main-results} are with post-aggregation and pre-aggregation, respectively.

\subsection{Listwise Methods}\label{app:listwise-methods}

The listwise methods are introduced in Section \ref{sec:listwise-methods}.

\section{Prompts}\label{app:prompts}

We present representative example prompts to illustrate the different settings. The prompts are adapted from MT-Bench \cite{zheng2023judging}. Auxiliary modifications are not shown, such as the prompt for second-turn evaluation in MT-Bench.

\subsection{Judgment Extraction Details}

To identify the token position containing the judgment, we use the specified format when available (e.g. ``Rating A: \texttt{\{rating\_a\}}.'' in pairwise scoring). Otherwise, we use the latest token position with more than 0.5 total probability assigned to judgment tokens. If no valid token is found, we default the judgment to the minimum score of 1 in the scoring setting, and to a tie in the ranking setting. (For Nectar experiments, we exclude instances with invalid silver-label judgments.)

For the local models (Llama-3.1-8B, Mistral-7B, Prometheus-2-7B) in no-CoT prompting, we force a prefix of the assistant's response (e.g. ``Rating A: '') and use a single output token as the judgment token position.

\subsection{Pointwise Prompts}\label{app:pointwise-prompts}

\begin{tcolorbox}[colback=white!5!white, colframe=black!75!black, width=\columnwidth, title={System prompt for pointwise scoring (\Magenta{CoT}, $K=9$)}]
\begin{PromptVerbatim}
Please act as an impartial judge and evaluate the quality of the response provided by an AI assistant to the user prompt displayed below. Your evaluation should consider factors such as the helpfulness, relevance, accuracy, depth, creativity, level of detail, and ethicality of the response. \Magenta{Begin your evaluation by providing a short explanation.} Be as objective as possible. \Magenta{After providing your explanation,} please rate the response with an integer score from 1 to 9, without further explanation.
\end{PromptVerbatim}
\label{pmp:sys-point-cot-k9}
\end{tcolorbox}

\begin{tcolorbox}[colback=white!5!white, colframe=black!75!black, width=\columnwidth, title={System prompt for pointwise scoring (\Magenta{no-CoT}, $K=9$)}]
\begin{PromptVerbatim}
Please act as an impartial judge and evaluate the quality of the response provided by an AI assistant to the user prompt displayed below. Your evaluation should consider factors such as the helpfulness, relevance, accuracy, depth, creativity, level of detail, and ethicality of the response. Be as objective as possible. Please rate the response with an integer score from 1 to 9, without further explanation.
\end{PromptVerbatim}
\label{pmp:sys-point-nocot-k9}
\end{tcolorbox}

\begin{tcolorbox}[colback=blue!5!white, colframe=blue!75!black, width=\columnwidth, title={User prompt for pointwise judgment}]
\begin{PromptVerbatim}
[User Prompt]
\VarField{User Prompt}
[End User Prompt]

[Start of Assistant's Answer]
\VarField{Assistant's Answer}
[End of Assistant's Answer]
\end{PromptVerbatim}
\label{pmp:user-point}
\end{tcolorbox}

\subsection{Pairwise Prompts}\label{app:pairwise-prompts}

\begin{tcolorbox}[colback=white!5!white, colframe=black!75!black, width=\columnwidth, title={System prompt for pairwise \Magenta{scoring} (CoT, $K=9$)}]
\begin{PromptVerbatim}
Please act as an impartial judge and evaluate the quality of the responses provided by two AI assistants to the user prompt displayed below. Your evaluation should consider factors such as the helpfulness, relevance, accuracy, depth, creativity, level of detail, and ethicality of their responses. Begin your evaluation by comparing the two responses and provide a short explanation. Avoid any position biases and ensure that the order in which the responses were presented does not influence your decision. Do not allow the length of the responses to influence your evaluation. Do not favor certain names of the assistants. Be as objective as possible. After providing your explanation, \Magenta{output your final verdict by strictly following this format: "Rating A: \{rating_a\}. Rating B: \{rating_b\}.", where "\{rating_a\}" and "\{rating_b\}" are integer scores from 1 to 9.}
\end{PromptVerbatim}
\label{pmp:sys-scalars-cot-k9}
\end{tcolorbox}

For pairwise ranking with the local models, we use a different prompt from the one below. We found that they would often fail to include the braces specified in the judgment format, so we omit them when prompting these models.

\begin{tcolorbox}[colback=white, colframe=black, width=\columnwidth, title={System prompt for pairwise \Magenta{ranking} (\Magenta{Likert-5}, CoT)}]
\begin{PromptVerbatim}
Please act as an impartial judge and evaluate the quality of the responses provided by two AI assistants to the user prompt displayed below. You should choose the assistant that follows the user's instructions and answers the user's question better. Your evaluation should consider factors such as the helpfulness, relevance, accuracy, depth, creativity, level of detail, and ethicality of their responses. Begin your evaluation by comparing the two responses and provide a short explanation. Avoid any position biases and ensure that the order in which the responses were presented does not influence your decision. Do not allow the length of the responses to influence your evaluation. Do not favor certain names of the assistants. Be as objective as possible. After providing your explanation, \Magenta{output your final verdict by strictly following this format: "[[>>]]" if assistant A is significantly better, "[[>]]" if assistant A is slightly better, "[[=]]" for a tie, "[[<]]" if assistant B is slightly better, and "[[<<]]" if assistant B is significantly better.}
\end{PromptVerbatim}
\label{pmp:sys-lik5-cot}
\end{tcolorbox}

\begin{tcolorbox}[colback=blue!5!white, colframe=blue!75!black, width=\columnwidth, title={User prompt for pairwise judgment}]
\begin{PromptVerbatim}
[User Prompt]
\VarField{User Prompt}
[End User Prompt]

[Start of Assistant A's Answer]
\VarField{Assistant A's Answer}
[End of Assistant A's Answer]

[Start of Assistant B's Answer]
\VarField{Assistant B's Answer}
[End of Assistant B's Answer]
\end{PromptVerbatim}
\label{pmp:user-pair}
\end{tcolorbox}

\subsection{Listwise Prompts}\label{app:listwise-prompts}

The listwise prompts are adapted from Nectar \cite{zhu2024starling}.

\begin{tcolorbox}[colback=white!5!white, colframe=black!75!black, width=\columnwidth, title={System prompt for listwise judgment ($N=7$), \Magenta{with intermediate pairwise preferences}}]
\begin{PromptVerbatim}
We are interested in ranking different large language model chat completions to a conversation. Please act as an impartial judge and evaluate the quality of the completions provided by the 7 AI assistants. Your evaluation should consider factors such as the helpfulness, relevance, accuracy, depth, creativity, level of detail, and ethicality of their responses.

\Magenta{After the conversation and assistant responses, the section "PAIRWISE EVALUATION ORDER" will specify the order in which to perform pairwise comparisons. Output an array in which, for each pairwise comparison, you choose the letter of the better response, or '=' for a tie. The array should be comma-separated and enclosed in double square brackets.}

\Magenta{Then, considering these pairwise rankings,} please rank all 7 responses from best to worst (breaking ties randomly), strictly in the following format: [[_, _, _, _, _, _, _]] where '_' contains an assistant's letter name.

Avoid any position biases and ensure that the order in which the responses were presented does not influence your decision. Do not allow the length of the responses to influence your evaluation. Do not favor certain names of the assistants. Be as objective as possible.
\end{PromptVerbatim}
\label{pmp:sys-list-interm}
\end{tcolorbox}

\begin{tcolorbox}[colback=white!5!white, colframe=black!75!black, width=\columnwidth, title={System prompt for listwise judgment ($N=7$), \Magenta{without intermediate pairwise preferences}}]
\begin{PromptVerbatim}
We are interested in ranking different large language model chat completions to a conversation. Please act as an impartial judge and evaluate the quality of the completions provided by the 7 AI assistants. Your evaluation should consider factors such as the helpfulness, relevance, accuracy, depth, creativity, level of detail, and ethicality of their responses.

Please rank all 7 responses from best to worst (breaking ties randomly), strictly in the following format: [[_, _, _, _, _, _, _]] where '_' contains an assistant's letter name.

Avoid any position biases and ensure that the order in which the responses were presented does not influence your decision. Do not allow the length of the responses to influence your evaluation. Do not favor certain names of the assistants. Be as objective as possible.
\end{PromptVerbatim}
\label{pmp:sys-list-direct}
\end{tcolorbox}

\begin{tcolorbox}[colback=blue!5!white, colframe=blue!75!black, width=\columnwidth, title={User prompt for listwise judgment ($N=7$). The presentation order is randomized. The pairwise evaluation order is randomized every instance for the prompt with intermediate pairwise preferences, and omitted for the prompt without intermediate pairwise preferences.}]
\begin{PromptVerbatim}
[CONVERSATION START]
\VarField{Conversation}
[CONVERSATION END]

[MODEL A RESPONSE START]
\VarField{Model A's response}
[MODEL A RESPONSE END]

[MODEL B RESPONSE START]
\VarField{Model B's response}
[MODEL B RESPONSE END]

[MODEL C RESPONSE START]
\VarField{Model C's response}
[MODEL C RESPONSE END]

[MODEL D RESPONSE START]
\VarField{Model D's response}
[MODEL D RESPONSE END]

[MODEL E RESPONSE START]
\VarField{Model E's response}
[MODEL E RESPONSE END]

[MODEL F RESPONSE START]
\VarField{Model F's response}
[MODEL F RESPONSE END]

[MODEL G RESPONSE START]
\VarField{Model G's response}
[MODEL G RESPONSE END]

PAIRWISE EVALUATION ORDER: [(G, C), (B, G), (C, D), (A, E), (G, A), (A, D), (B, A), (B, E), (B, F), (A, C), (E, C), (E, F), (B, D), (F, A), (G, E), (F, C), (F, D), (C, B), (F, G), (D, G), (E, D)]
\end{PromptVerbatim}
\label{pmp:user-list}
\end{tcolorbox}

\section{Datasets}\label{app:datasets}

RewardBench \cite{lambert2024rewardbench} is a reward model benchmark spanning chat, reasoning, and safety. Each instance consists of a prompt, a chosen response, and a rejected response, all manually verified. The dataset categories are \textit{Chat}, with 358 instances sourced from AlpacaEval \cite{alpaca_eval} and MT-Bench \cite{zheng2023judging}; \textit{Chat Hard}, with 456 instances sourced from MT-Bench and LLMBar \cite{zeng2023evaluating}; \textit{Safety}, with 740 instances sourced from XSTest \cite{Rttger2023XSTestAT}, Do-Not-Answer \cite{Wang2023DoNotAnswerAD}, and original data; and \textit{Reasoning}, with 1431 instances sourced from PRM800k \cite{Lightman2023LetsVS} and HumanEvalPack \cite{Muennighoff2023OctoPackIT}. Except for excluding the prior sets category, we follow the original work and compute the final score as the average of the category scores.

MT-Bench \cite{zheng2023judging} is a dataset of multi-turn questions spanning writing, roleplay, extraction, reasoning, math, coding, knowledge I (STEM), and knowledge II (humanities/social science). There are 3,355 (prompt, model pair, human judge, turn) tuples, 1,814 unique (prompt, model pair, turn) tuples, and 80 unique prompts each with two turns of interaction. To evaluate accuracy, we use the 1,132 instances with unanimous non-tie human judgments. To evaluate MSE, we use all 1,814 instances and set the label of an instance to the average of the human judgments, where a 0 or 1 represents the evaluated winner, and a 0.5 represents a tie.

Nectar \cite{zhu2024starling} is a dataset of 183k prompts each with 7 model responses. The prompts are sourced from Anthropic-HH \cite{Bai2022TrainingAH}, LMSYS-Chat-1M \cite{Zheng2023LMSYSChat1MAL}, UltraFeedback \cite{Cui2023UltraFeedbackBL}, and ShareGPT. We use a random subset of size 1,000.

RM-Bench \cite{Liu2024RMBenchBR} is a reward model benchmark focusing on sensitivity to subtle content differences and resistance to style biases. There are 1,327 instances spanning chat, code, math, and safety. Similar to RewardBench, we follow the original work and average the 4 category scores. For each prompt, there are 3 pairs of (chosen, rejected) responses, where each pair is written with a particular style regarding concision and whether formatted as plain text or markdown.

The HelpSteer2 dataset \cite{Wang2024HelpSteer2OD} contains multiple human ratings on a 0-4 scale for five attributes (helpfulness, correctness, coherence, complexity, verbosity) for each (prompt, response) instance. We use a random subset of size 1,000.

\subsection{Listwise Evaluation}\label{app:list-eval}

For the listwise setting, we use the same evaluation setup as with the pointwise and pairwise setting.\footnote{This means that our MT-Bench results are directly comparable across settings.} We concern ourselves with agreement at the pair level rather than the list level because pairwise preferences are sufficient to produce a total order, such as by choosing the maximum likelihood order \cite{liu2024aligning, Liusie2024EfficientLC} or with graph-theoretic methods \cite{tideman, schulze, Li2024TSPRankBP}. Thus, pairwise preferences are an adequate unit at which to measure agreement, and the aggregation into a total order may be modularized away for experimental simplicity.

To compute accuracy on Nectar with silver labels (Section \ref{sec:listwise-experimental-setup}), we take the sign of the silver label as the silver label for accuracy.

\section{Additional Results}\label{app:add-res}

Table \ref{tab:pointwise-pairwise-main-results} is an expanded version of Table \ref{tab:pointwise-pairwise-short-main-results}, providing a subset breakdown of RewardBench. We observe particularly large gains for pointwise scoring on the Reasoning subset, e.g. absolute +7.7\% and +17.1\% for GPT-4o and Llama-3.1-8B.

\begin{table*}[t]
\small
\centering
\begin{adjustbox}{center=\columnwidth}
\begin{tabular}{ccccccccc}
\toprule
\multirow{2}{*}{Model}& \multirow{2}{*}{Setting} & \multirow{2}{*}{Method} & \multicolumn{5}{c}{RewardBench} & \multirow{2}{*}{MT-Bench} \\\cmidrule(lr){4-8}
 & & &Chat & Chat Hard & Safety & Reasoning & Total &\\
\midrule
\multirow{6}{*}{GPT-4o} & \multirow{2}{*}{point score} & mode & 95.8, 89.7 & 76.0, 77.4 & 89.3, 88.5 & 79.5, 80.3 & 85.1, 84.0 & 81.9, 80.5 \\
 &  & mean & \textbf{97.1}, 94.3 & 75.2, \textbf{79.8} & \textbf{90.3}, \underline{89.7} & \underline{87.0}, \textbf{88.0} & \underline{87.4}, \textbf{88.0} & \textbf{83.6}, \underline{83.2} \\ \cline{2-9}
 & \multirow{2}{*}{pair score} & mode & \underline{97.3}, \textbf{97.9} & 69.0, \underline{70.7} & \underline{89.1}, \textbf{89.5} & 91.3, 91.3 & 86.7, \underline{87.4} & \underline{86.2}, \underline{86.5} \\
 &  & mean & \underline{97.2}, \underline{97.8} & \underline{69.7}, \textbf{70.8} & \textbf{89.5}, \textbf{89.5} & \underline{91.9}, \textbf{92.4} & \underline{87.1}, \textbf{87.6} & \underline{86.3}, \textbf{86.8} \\ \cline{2-9}
 & \multirow{2}{*}{pair rank} & mode & 96.9, \underline{97.6} & 76.4, \underline{79.1} & 89.0, \textbf{90.9} & 91.4, 91.3 & 88.4, 89.7 & 86.3, 85.6 \\
 &  & mean & 96.2, \textbf{98.3} & 76.6, \textbf{79.4} & 88.5, \underline{90.8} & \underline{93.0}, \textbf{93.6} & 88.6, \textbf{90.5} & \textbf{87.3}, 85.9 \\ \hline
\multirow{6}{*}{Llama-3.1-8B} & \multirow{2}{*}{point score} & mode & 83.8, 87.6 & \underline{57.6}, \underline{58.0} & 76.2, 78.2 & 60.8, 64.8 & 69.6, 72.2 & 74.9, 71.9 \\
 &  & mean & 89.0, \textbf{95.8} & \underline{58.6}, \textbf{58.8} & 73.0, \textbf{80.8} & 70.2, \textbf{81.9} & 72.7, \textbf{79.3} & 78.7, \textbf{81.5} \\ \cline{2-9}
 & \multirow{2}{*}{pair score} & mode & 92.0, 94.3 & \textbf{45.4}, \textbf{45.4} & 69.5, \underline{78.8} & 79.9, 82.4 & 71.7, 75.2 & \textbf{82.6}, \underline{82.4} \\
 &  & mean & 92.6, \textbf{95.8} & \underline{44.6}, \underline{45.0} & 69.3, \textbf{78.9} & 81.7, \textbf{87.6} & 72.1, \textbf{76.8} & \underline{82.3}, 81.2 \\ \cline{2-9}
 & \multirow{2}{*}{pair rank} & mode & 76.7, 65.2 & \textbf{52.3}, 48.1 & 71.0, 66.4 & 75.6, 55.8 & 68.9, 58.9 & 76.2, 63.0 \\
 &  & mean & \underline{90.5}, \textbf{93.0} & \underline{50.0}, 44.1 & \textbf{78.1}, 72.7 & \textbf{78.3}, 64.6 & \textbf{74.2}, 68.6 & \textbf{80.0}, 76.5 \\ \hline
\multirow{6}{*}{Mistral-7B} & \multirow{2}{*}{point score} & mode & 52.4, 66.2 & \underline{51.5}, \underline{50.5} & \textbf{79.9}, 75.7 & 57.8, 58.4 & 60.4, 62.7 & 59.5, 66.2 \\
 &  & mean & 54.5, \textbf{82.1} & \textbf{53.5}, \underline{49.1} & \textbf{79.9}, \underline{79.6} & 67.2, \textbf{77.5} & 63.8, \textbf{72.1} & 62.6, \textbf{74.0} \\ \cline{2-9}
 & \multirow{2}{*}{pair score} & mode & 87.6, \underline{89.9} & \underline{40.2}, \underline{40.4} & \underline{74.0}, \underline{73.0} & 67.4, 72.4 & 67.3, 68.9 & \underline{79.3}, \underline{79.8} \\
 &  & mean & \underline{89.2}, \textbf{91.1} & \textbf{41.2}, \underline{39.3} & \textbf{74.1}, \underline{73.4} & 67.8, \textbf{80.2} & 68.1, \textbf{71.0} & \underline{80.0}, \textbf{80.4} \\ \cline{2-9}
 & \multirow{2}{*}{pair rank} & mode & 51.0, 51.5 & \textbf{51.0}, 46.2 & 62.2, 66.8 & 61.0, 50.8 & 56.3, 53.8 & 51.5, 51.5 \\
 &  & mean & \underline{79.5}, \textbf{81.7} & 39.3, 36.3 & \textbf{73.1}, 67.7 & \textbf{63.8}, 50.6 & \textbf{63.9}, 59.1 & \textbf{73.5}, 65.5 \\ \hline
\multirow{6}{*}{Prometheus-2-7B} & \multirow{2}{*}{point score} & mode & 81.3, 81.7 & 50.5, 50.8 & 65.9, 73.4 & 59.2, 58.2 & 64.3, 66.0 & 72.5, 73.5 \\
 &  & mean & 82.4, \textbf{92.2} & 48.9, \textbf{54.4} & 65.7, \textbf{76.6} & 61.3, \textbf{77.6} & 64.6, \textbf{75.2} & 72.1, \textbf{81.6} \\ \cline{2-9}
 & \multirow{2}{*}{pair score} & mode & \underline{91.2}, \underline{92.0} & \textbf{44.1}, \underline{43.6} & \textbf{75.9}, 69.4 & 72.7, 69.6 & \textbf{71.0}, 68.7 & 78.4, \underline{80.8} \\
 &  & mean & \underline{91.3}, \textbf{93.0} & 42.7, \underline{43.0} & 74.9, 72.0 & \underline{73.0}, \textbf{75.1} & 70.5, \underline{70.8} & 78.3, \textbf{80.9} \\ \cline{2-9}
 & \multirow{2}{*}{pair rank} & mode & 55.6, 45.4 & \textbf{51.0}, \underline{50.0} & 66.6, 49.7 & 65.3, 47.8 & 59.6, 48.2 & 51.5, 43.0 \\
 &  & mean & \textbf{90.5}, 45.0 & 44.3, \underline{50.7} & \textbf{74.2}, 55.5 & \textbf{69.8}, 44.1 & \textbf{69.7}, 48.8 & \textbf{75.4}, 33.4 \\
\bottomrule
\end{tabular}
\end{adjustbox}
\caption{Mode vs. mean and CoT vs. no-CoT (comma-separated) accuracy results (\%). Expanded version of Table \ref{tab:pointwise-pairwise-short-main-results}.
}
\label{tab:pointwise-pairwise-main-results}
\end{table*}

Tables \ref{tab:pointwise-method-results} ($K=9$) and \ref{tab:pointwise-method-k99-results} ($K=99$) show pointwise results over methods (expanded versions of Tables \ref{tab:pointwise-method-nocot-results} and \ref{tab:pointwise-method-k99-nocot-results}).
Tables \ref{tab:pointwise-tie-rewardbench-results} and \ref{tab:pointwise-tie-mtbench-results} show expanded tie analyses on RewardBench (simplified in Table \ref{tab:pointwise-tie-rewardbench-small-results}) and MT-Bench.

Table \ref{tab:pairwise-method-results} shows pairwise ranking results over methods, extending Table \ref{tab:pairwise-method-short-results}. Table \ref{tab:pairwise-vspace-results} compares the Likert scales used for pairwise ranking, extending Table \ref{tab:pairwise-vspace-short-results}.
In Table \ref{tab:pairwise-vspace-results}, the most calibrated setting on MT-Bench is (GPT-4o) Likert-5 no-CoT, achieving a 31\% lower MSE than the most accurate setting, Likert-2 CoT, suggesting that a finer granularity has potential to improve calibration \cite{Liu2024RewardMW}.
With GPT-4o in Tables \ref{tab:pairwise-method-results} and \ref{tab:pairwise-vspace-results}, MSE is always minimized with no-CoT, highlighting the discord between CoT's sharpening effect and calibration. This result is in line with AlpacaEval \cite{dubois2024length}, which uses no-CoT and judgment probabilities, but deviating from WB-Reward and Arena-Hard-Auto \cite{lin2024wildbench, li2024crowdsourced}, which use CoT and decoded judgments.

\begin{table*}[t]
\small
\centering
\begin{tabular}{cccccc}
\toprule
\multirow{2}{*}{Model} & \multirow{2}{*}{Method} & \multicolumn{2}{c}{RewardBench} & \multicolumn{2}{c}{MT-Bench}\\\cmidrule(lr){3-4}\cmidrule(lr){5-6}
&& Acc $\uparrow$ & MSE $\downarrow$ & Acc $\uparrow$ & MSE $\downarrow$ \\
\midrule
\multirow{8}{*}{GPT-4o} & \sc mode & 85.1, 84.0 & .116, .118 & 81.9, 80.5 & .152, .145 \\
 & \sc mean & \underline{87.4}, 88.0 & \underline{.099}, .102 & 83.6, \underline{83.2} & .115, \underline{.097} \\
 & \sc [mean] & 85.1, 85.2 & .116, .109 & 82.0, 80.2 & .150, .146 \\
 & \sc medi & 85.0, 84.6 & .116, .112 & 82.0, 80.2 & .150, .142 \\
 & \sc 1p & 84.8, 84.3 & .120, .116 & 82.6, 81.0 & .141, .138 \\
 & \sc ram & \underline{87.4}, \textbf{88.4} & \underline{.099}, .100 & \textbf{83.9}, \underline{83.4} & .115, \textbf{.096} \\
 & \sc qt & \underline{87.4}, 87.9 & .107, \textbf{.096} & 83.5, \underline{83.2} & .139, .118 \\
 & \sc ps & \underline{87.4}, 87.8 & .106, \textbf{.096} & 83.5, \underline{83.3} & .136, .103 \\ \hline
\multirow{8}{*}{\shortstack{Llama\\-3.1-8B}} & \sc mode & 69.6, 72.2 & .237, .192 & 74.9, 71.9 & .177, .142 \\
 & \sc mean & 72.7, 79.3 & .198, .155 & 78.7, \textbf{81.5} & .129, .104 \\
 & \sc [mean] & 70.1, 75.0 & .238, .186 & 75.7, 75.0 & .172, .145 \\
 & \sc medi & 69.8, 73.6 & .238, .191 & 75.2, 73.9 & .176, .142 \\
 & \sc 1p & 70.2, 76.0 & .238, .183 & 76.8, 79.2 & .172, .147 \\
 & \sc ram & 72.7, \textbf{79.9} & .200, \textbf{.152} & 78.8, \underline{81.4} & .130, \textbf{.102} \\
 & \sc qt & 72.8, 79.0 & .220, .164 & 78.7, 81.1 & .154, .116 \\
 & \sc ps & 72.8, 78.9 & .216, .161 & 78.6, \underline{81.4} & .149, .110 \\
\bottomrule
\end{tabular}
\caption{Pointwise results over methods. Comma-separated values are with and without CoT (expanded version of Table \ref{tab:pointwise-method-nocot-results}). Text styling follows Table \ref{tab:pointwise-pairwise-short-main-results}.
}
\label{tab:pointwise-method-results}
\end{table*}

\begin{table*}[t]
\small
\centering
\setlength{\tabcolsep}{4pt}
\begin{adjustbox}{center=\columnwidth}
\begin{tabular}{cccccccc}
\toprule
\multirow{2}{*}{Model} & \multirow{2}{*}{Method} & \multicolumn{2}{c}{Tie rate} & \multicolumn{2}{c}{{\sc mean}'s accuracy} & \multicolumn{2}{c}{Non-tie accuracy $\Delta$ $\uparrow$} \\\cmidrule(lr){3-4}\cmidrule(lr){5-6}\cmidrule(lr){7-8}
 & & $K=9$ & $K=99$ & $K=9$ & $K=99$ & $K=9$ & $K=99$ \\
\midrule
\multirow{4}{*}{GPT-4o} & \sc mode & .13, .17 & .09, .20 & 64, 72 & 61, 73 & +0.0, --0.1 & +0.0, --0.3 \\
 & \sc [mean] & .13, .16 & .02, .03 & 65, 67 & 61, 53 & +0.0, +0.0 & +0.0, --0.0 \\
 & \sc medi & .13, .17 & .06, .09 & 65, 70 & 58, 62 & --0.0, +0.0 & --0.0, --0.1 \\
 & \sc 1p & .13, .16 & .05, .08 & 66, 66 & 58, 60 & +0.1, +0.1 & +0.0, +0.5 \\ \hline
\multirow{4}{*}{\shortstack{Llama\\-3.1-8B}} & \sc mode & .27, .35 & .18, .24 & 60, 69 & 63, 70 & +0.2, --0.2 & +0.2, --2.0 \\
 & \sc [mean] & .25, .26 & .07, .07 & 58, 64 & 56, 61 & +0.0, +0.0 & +0.0, +0.0 \\
 & \sc medi & .26, .29 & .11, .11 & 60, 67 & 59, 67 & +0.1, --0.5 & --0.2, --0.6 \\
 & \sc 1p & .24, .23 & .08, .08 & 61, 65 & 55, 57 & +0.3, +0.6 & +0.8, +0.1 \\
\bottomrule
\end{tabular}
\end{adjustbox}
\caption{Tie analysis for discrete pointwise methods on RewardBench (expanded version of Table \ref{tab:pointwise-tie-rewardbench-small-results}). Tie rate is the proportion of instances where the method predicts a tie, over which we report {\sc mean}'s accuracy (\%); excess of 50\% or 75\% indicates room for improving accuracy or MSE, respectively. Non-tie accuracy $\Delta$ (\%) is the method's accuracy minus {\sc mean}'s accuracy over the non-tie instances. Comma-separated values are with and without CoT. We find that the mode has the most ties, the highest {\sc mean} accuracy, and the lowest non-tie accuracy delta (i.e. poor recall without better precision), especially for no-CoT $K=99$.
}
\label{tab:pointwise-tie-rewardbench-results}
\end{table*}

\begin{table*}[t]
\small
\centering
\setlength{\tabcolsep}{4pt}
\begin{adjustbox}{center=\columnwidth}
\begin{tabular}{cccccccc}
\toprule
\multirow{2}{*}{Model} & \multirow{2}{*}{Method} & \multicolumn{2}{c}{Tie rate} & \multicolumn{2}{c}{{\sc mean}'s accuracy} & \multicolumn{2}{c}{Non-tie accuracy $\Delta$ $\uparrow$} \\\cmidrule(lr){3-4}\cmidrule(lr){5-6}\cmidrule(lr){7-8}
 & & $K=9$ & $K=99$ & $K=9$ & $K=99$ & $K=9$ & $K=99$ \\
\midrule
\multirow{4}{*}{GPT-4o} & \sc mode & .13, .21 & .08, .22 & 62, 62 & 64, 67 & +0.0, --0.2 & +0.1, --0.7 \\
 & \sc [mean] & .13, .21 & .02, .04 & 61, 64 & 42, 48 & +0.0, +0.0 & +0.0, +0.0 \\
 & \sc medi & .13, .22 & .05, .11 & 61, 64 & 64, 55 & +0.0, +0.1 & +0.0, --0.6 \\
 & \sc 1p & .14, .19 & .06, .09 & 56, 62 & 67, 56 & --0.1, +0.1 & +0.3, +0.7 \\ \hline
\multirow{4}{*}{\shortstack{Llama\\-3.1-8B}} & \sc mode & .25, .45 & .14, .26 & 65, 71 & 61, 65 & --0.1, --1.0 & --0.4, --3.0 \\
 & \sc [mean] & .24, .36 & .06, .09 & 63, 68 & 49, 55 & +0.0, +0.0 & +0.0, --0.1 \\
 & \sc medi & .25, .40 & .10, .18 & 65, 69 & 54, 58 & +0.0, --0.3 & --0.4, +0.5 \\
 & \sc 1p & .20, .23 & .07, .07 & 60, 59 & 53, 50 & +0.1, --0.3 & +1.8, +0.3 \\
\bottomrule
\end{tabular}
\end{adjustbox}
\caption{Tie analysis for discrete pointwise methods on MT-Bench, mirroring Table \ref{tab:pointwise-tie-rewardbench-results}.}
\label{tab:pointwise-tie-mtbench-results}
\end{table*}

\begin{table*}[t]
\small
\centering
\begin{tabular}{cccccc}
\toprule
\multirow{2}{*}{Model} & \multirow{2}{*}{Method} & \multicolumn{2}{c}{RewardBench} & \multicolumn{2}{c}{MT-Bench}\\\cmidrule(lr){3-4}\cmidrule(lr){5-6}
&& Acc $\uparrow$ & MSE $\downarrow$ & Acc $\uparrow$ & MSE $\downarrow$ \\
\midrule
\multirow{8}{*}{GPT-4o} & \sc mode & 86.1$_\text{+1.0}$, 81.7$_\text{--2.3}$ & .118$_\text{+.002}$, .134$_\text{+.016}$ & 83.8$_\text{+1.9}$, 78.4$_\text{--2.1}$ & .152$_\text{+.000}$, .158$_\text{+.013}$ \\
 & \sc mean & \textbf{87.4}$_\text{+0.0}$, \underline{86.7}$_\text{--1.3}$ & \textbf{.097}$_\text{--.002}$, .108$_\text{+.006}$ & \underline{84.8}$_\text{+1.2}$, 82.9$_\text{--0.3}$ & .105$_\text{--.010}$, .099$_\text{+.002}$ \\
 & \sc [mean] & 87.0$_\text{+1.9}$, \underline{86.5}$_\text{+1.3}$ & .124$_\text{+.008}$, .127$_\text{+.018}$ & \textbf{85.1}$_\text{+3.1}$, 82.7$_\text{+2.5}$ & .167$_\text{+.017}$, .182$_\text{+.036}$ \\
 & \sc medi & 86.7$_\text{+1.7}$, 85.2$_\text{+0.6}$ & .119$_\text{+.003}$, .126$_\text{+.014}$ & 84.1$_\text{+2.1}$, 81.5$_\text{+1.3}$ & .160$_\text{+.010}$, .170$_\text{+.028}$ \\
 & \sc 1p & 86.6$_\text{+1.8}$, 86.4$_\text{+2.1}$ & .121$_\text{+.001}$, .116$_\text{+.000}$ & 84.2$_\text{+1.6}$, 82.7$_\text{+1.7}$ & .159$_\text{+.018}$, .165$_\text{+.027}$ \\
 & \sc ram & 87.1$_\text{--0.3}$, \underline{86.7}$_\text{--1.7}$ & .098$_\text{--.001}$, .104$_\text{+.004}$ & \textbf{85.1}$_\text{+1.2}$, 83.0$_\text{--0.4}$ & .106$_\text{--.009}$, \textbf{.098}$_\text{+.002}$ \\
 & \sc qt & \underline{87.3}$_\text{--0.1}$, \underline{86.6}$_\text{--1.3}$ & .112$_\text{+.005}$, .114$_\text{+.018}$ & \underline{84.8}$_\text{+1.3}$, 82.7$_\text{--0.5}$ & .149$_\text{+.010}$, .147$_\text{+.029}$ \\
 & \sc ps & \underline{87.3}$_\text{--0.1}$, \underline{86.6}$_\text{--1.2}$ & .105$_\text{--.001}$, .105$_\text{+.009}$ & \underline{84.8}$_\text{+1.3}$, 82.4$_\text{--0.9}$ & .130$_\text{--.006}$, .107$_\text{+.004}$ \\ \hline
\multirow{8}{*}{\shortstack{Llama\\-3.1-8B}} & \sc mode & 73.4$_\text{+3.8}$, 72.0$_\text{--0.2}$ & .222$_\text{--.015}$, .221$_\text{+.029}$ & 77.3$_\text{+2.4}$, 75.1$_\text{+3.2}$ & .191$_\text{+.014}$, .169$_\text{+.027}$ \\
 & \sc mean & 75.9$_\text{+3.2}$, \underline{79.3}$_\text{+0.0}$ & .183$_\text{--.015}$, .156$_\text{+.001}$ & 79.3$_\text{+0.6}$, \underline{81.3}$_\text{--0.2}$ & .125$_\text{--.004}$, \underline{.103}$_\text{--.001}$ \\
 & \sc [mean] & 75.3$_\text{+5.2}$, 78.5$_\text{+3.5}$ & .229$_\text{--.009}$, .198$_\text{+.012}$ & 79.3$_\text{+3.6}$, 80.7$_\text{+5.7}$ & .201$_\text{+.029}$, .180$_\text{+.035}$ \\
 & \sc medi & 74.4$_\text{+4.6}$, 76.5$_\text{+2.9}$ & .228$_\text{--.010}$, .207$_\text{+.016}$ & 78.4$_\text{+3.2}$, 80.1$_\text{+6.2}$ & .198$_\text{+.022}$, .161$_\text{+.019}$ \\
 & \sc 1p & 76.2$_\text{+6.0}$, 78.5$_\text{+2.5}$ & .218$_\text{--.020}$, .195$_\text{+.012}$ & \underline{80.6}$_\text{+3.8}$, \underline{81.5}$_\text{+2.3}$ & .187$_\text{+.015}$, .177$_\text{+.030}$ \\
 & \sc ram & 76.1$_\text{+3.4}$, \textbf{79.7}$_\text{--0.2}$ & .179$_\text{--.021}$, \textbf{.152}$_\text{+.000}$ & 79.7$_\text{+0.9}$, \underline{81.1}$_\text{--0.3}$ & .123$_\text{--.007}$, \textbf{.102}$_\text{+.000}$ \\
 & \sc qt & 75.7$_\text{+2.9}$, 78.7$_\text{--0.3}$ & .214$_\text{--.006}$, .177$_\text{+.013}$ & 78.8$_\text{+0.1}$, 81.3$_\text{+0.2}$ & .179$_\text{+.025}$, .143$_\text{+.027}$ \\
 & \sc ps & 75.7$_\text{+2.9}$, 78.6$_\text{--0.3}$ & .203$_\text{--.013}$, .163$_\text{+.002}$ & 78.6$_\text{+0.0}$, \textbf{81.8}$_\text{+0.4}$ & .151$_\text{+.002}$, .111$_\text{+.001}$ \\
\bottomrule
\end{tabular}
\caption{Pointwise results over methods ($K=99$). Comma-separated values are with and without CoT (expanded version of Table \ref{tab:pointwise-method-k99-nocot-results}). Subscripts denote change from $K=9$ (Table \ref{tab:pointwise-method-results}). Text styling follows Table \ref{tab:pointwise-pairwise-short-main-results}.
}
\label{tab:pointwise-method-k99-results}
\end{table*}

\begin{table*}[t]
\small
\centering
\begin{adjustbox}{center=\columnwidth}
\begin{tabular}{ccccccc}
\toprule
\multirow{2}{*}{Model} & \multirow{2}{*}{Center} & \multirow{2}{*}{\shortstack{Agg.\\Time}} & \multicolumn{2}{c}{RewardBench} & \multicolumn{2}{c}{MT-Bench} \\\cmidrule(lr){4-5}\cmidrule(lr){6-7}
 & & & Acc $\uparrow$ & MSE $\downarrow$ & Acc $\uparrow$ & MSE $\downarrow$ \\
\midrule
\multirow{6}{*}{GPT-4o} & \multirow{2}{*}{mode} & post & 88.1, 89.3 & .099, \textbf{.090} & \underline{86.1}, 84.9 & \textbf{.139}, \underline{.142} \\
 &  & pre & 88.4, \textbf{90.3} & .112, .094 & \textbf{86.5}, 85.2 & .154, .154 \\ \cline{2-7}
 & \multirow{2}{*}{median} & post & 88.1, 89.3 & .099, \textbf{.091} & \underline{86.1}, 84.9 & \textbf{.138}, \underline{.142} \\
 &  & pre & 88.4, \textbf{90.0} & .111, \underline{.094} & \textbf{86.6}, 85.4 & .153, .146 \\ \cline{2-7}
 & \multirow{2}{*}{mean} & post & 88.9, \textbf{90.4} & .098, \textbf{.077} & \textbf{86.5}, 85.4 & .132, .100 \\
 &  & pre & 88.9, \textbf{90.4} & .098, \underline{.078} & \textbf{86.6}, 85.4 & .132, \textbf{.097} \\ \hline
\multirow{6}{*}{\shortstack{Llama\\-3.1-8B}} & \multirow{2}{*}{mode} & post & 56.7, 52.4 & \textbf{.240}, .279 & 57.5, 53.4 & .192, \textbf{.176} \\
 &  & pre & \textbf{73.1}, 66.1 & .265, .337 & \textbf{78.1}, 70.9 & .222, .268 \\ \cline{2-7}
 & \multirow{2}{*}{median} & post & 56.8, 52.5 & \textbf{.240}, .279 & 57.5, 53.5 & .192, \textbf{.176} \\
 &  & pre & \textbf{72.9}, 65.3 & .261, .319 & \textbf{78.0}, 69.1 & .218, .238 \\ \cline{2-7}
 & \multirow{2}{*}{mean} & post & \underline{73.2}, 65.6 & \textbf{.207}, .229 & \textbf{78.2}, 70.5 & \textbf{.144}, \underline{.146} \\
 &  & pre & \textbf{73.2}, 66.3 & .222, .240 & \underline{78.1}, 70.8 & .155, .155 \\
\bottomrule
\end{tabular}
\end{adjustbox}
\caption{Pairwise ranking results over methods, using Likert-3 (expanded version of Table \ref{tab:pairwise-method-short-results}). Comma-separated values are with and without CoT. Text styling follows Table \ref{tab:pointwise-pairwise-short-main-results}.}
\label{tab:pairwise-method-results}
\end{table*}

\begin{table}[t]
\small
\centering
\setlength{\tabcolsep}{4pt}
\begin{adjustbox}{center=\columnwidth}
\begin{tabular}{cccccc}
\toprule
\multirow{2}{*}{Model} & \multirow{2}{*}{$K$} & \multicolumn{2}{c}{RewardBench} & \multicolumn{2}{c}{MT-Bench} \\\cmidrule(lr){3-4}\cmidrule(lr){5-6}
 & & Acc $\uparrow$ & MSE $\downarrow$ & Acc $\uparrow$ & MSE $\downarrow$ \\
\midrule
\multirow{3}{*}{GPT-4o} & 2 & 88.6, \textbf{90.5} & .094, \textbf{.077} & \textbf{87.3}, 85.9 & .136, .101 \\
 & 3 & 88.9, \underline{90.4} & .098, \textbf{.078} & \underline{86.6}, 85.4 & .132, .097 \\
 & 5 & 88.8, 89.5 & .099, .106 & 84.7, 85.8 & .129, \textbf{.087} \\ \hline
\multirow{3}{*}{\shortstack{Llama\\-3.1-8B}} & 2 & \textbf{74.2}, 68.6 & \textbf{.187}, .214 & \textbf{80.0}, 76.5 & \textbf{.126}, .135 \\
 & 3 & \underline{73.2}, 66.3 & .222, .240 & 78.1, 70.8 & .155, .155 \\
 & 5 & 70.0, 58.5 & .215, .234 & 77.1, 64.8 & .142, .153 \\
\bottomrule
\end{tabular}
\end{adjustbox}
\caption{Pairwise ranking results over Likert-$K$ scales, using pre-aggregation mean (expanded version of Table \ref{tab:pairwise-vspace-short-results}). Comma-separated values are with and without CoT. Text styling follows Table \ref{tab:pointwise-pairwise-short-main-results}.}
\label{tab:pairwise-vspace-results}
\end{table}

\subsection{DeepSeek-V3 Results}
\label{sec:deepseek-results}

We provide partial results for DeepSeek-V3 \cite{deepseekai2025deepseekv3technicalreport}, a model of comparable size to GPT-4o. Tables \ref{tab:deepseek-pointwise} and \ref{tab:deepseek-listwise} contain pointwise and listwise results, respectively. The trends for DeepSeek-V3 match those of GPT-4o.

\begin{table}[t]
\small
\centering
\begin{tabular}{ccc}
\toprule
Method & Acc $\uparrow$ & MSE $\downarrow$ \\
\midrule
{\sc mode} & \underline{84.7}, 82.5 & \underline{0.123}, 0.128\\
{\sc mean} & \underline{84.8}, \underline{84.2} & \underline{0.119}, \underline{0.120}\\
{\sc [mean]} & \underline{84.7}, 82.7 & \underline{0.123}, 0.127\\
{\sc medi} & \underline{84.7}, 82.6 & \underline{0.123}, 0.128\\
{\sc 1p} & 84.5, 82.9 & \underline{0.124}, 0.126\\
{\sc ram} & \textbf{84.9}, \underline{84.1} & \underline{0.120}, \textbf{0.118}\\
{\sc qt} & \textbf{85.0}, 83.9 & \underline{0.122}, 0.125\\
{\sc ps} & \textbf{85.0}, 83.9 & \underline{0.122}, 0.125\\
\bottomrule
\end{tabular}
\caption{Pointwise results with DeepSeek-V3 on RewardBench. $K=9$. Comma-separated values are with and without CoT. Text styling follows Table \ref{tab:pointwise-pairwise-short-main-results}.}
\label{tab:deepseek-pointwise}
\end{table}

\begin{table}[t]
\small
\centering
\begin{tabular}{cccccc}
\toprule
\multirow{2}{*}{Space} & \multirow{2}{*}{Method} & \multicolumn{2}{c}{Nectar} & \multicolumn{2}{c}{RM-Bench} \\\cmidrule(lr){3-4}\cmidrule(lr){5-6}
 & & Acc & MSE & Acc & MSE \\
\midrule
direct list & mode & 83.6 & 0.149 & \textbf{67.8} & 0.322 \\
direct list & mean & \textbf{84.0} & \textbf{0.129} & \underline{67.6} & \textbf{0.307} \\
\bottomrule
\end{tabular}
\caption{Listwise results with DeepSeek-V3. Text styling follows Table \ref{tab:pointwise-pairwise-short-main-results}.}
\label{tab:deepseek-listwise}
\end{table}

\section{Analysis}

\begin{table}[t]
\small
\centering
\begin{tabular}{ccccc}
\toprule
Model & Setting & $K$ & RewardBench & MT-Bench \\
\midrule
\multirow{4}{*}{GPT-4o} & point score & 9 & .000, .008 & .000, .012 \\
 & point score & 99 & .362, .409 & .357, .440 \\
 & pair rank & 3 & .000, .018 & .000, .019 \\
 & pair rank & 5 & .014, .049 & .021, .041 \\ \hline
\multirow{4}{*}{\shortstack{Llama\\-3.1-8B}} & point score & 9 & .009, .040 & .013, .025 \\
 & point score & 99 & .356, .379 & .382, .365 \\
 & pair rank & 3 & .044, .091 & .051, .081 \\
 & pair rank & 5 & .107, .194 & .107, .245\\
\bottomrule
\end{tabular}
\caption{A study on multimodality (see Appendix \ref{app:multimodality}). Comma-separated values are with and without CoT.}
\label{tab:multimodality}
\end{table}

\begin{table}[t]
\small
\centering
\begin{tabular}{ccc}
\toprule
Model & Setting & MT-Bench \\
\midrule
\multirow{3}{*}{GPT-4o} & point score & \textbf{+0.21}, \textbf{+0.24} \\
 & pair score & \textbf{+0.19}, \textbf{+0.27} \\
 & pair rank & \textbf{+0.19}, \textbf{+0.27} \\ \hline
\multirow{3}{*}{\shortstack{Llama\\-3.1-8B}} & point score & \textbf{+0.21}, \textbf{+0.14} \\
 & pair score & \textbf{+0.20}, \textbf{+0.24} \\
 & pair rank & +0.02, --0.04 \\
\bottomrule
\end{tabular}
\caption{Spearman's $\rho$ between standard deviation of human judgments and that of LLM's judgment distribution. Comma-separated values are with and without CoT. Bold denotes significant correlation ($\alpha = 0.01$). Ranking uses Likert-3; scoring uses $K=9$ converted to a Likert-3 distribution $[P(X_1 > X_2), P(X_1 = X_2), P(X_1 < X_2)]$.}
\label{tab:mtbench-agreement}
\end{table}

\begin{table*}[t]
\small
\centering
\begin{tabular}{cccccc}
\toprule
Model & Helpfulness & Correctness & Coherence & Complexity & Verbosity \\
\midrule
GPT-4o & \textbf{+0.24} & \textbf{+0.36} & \textbf{+0.32} & +0.02 & --0.01 \\
Llama-3.1-8B & \textbf{+0.14} & \textbf{+0.22} & \textbf{+0.22} & --0.00 & +0.01 \\
\bottomrule
\end{tabular}
\caption{Spearman's $\rho$ between standard deviation of human judgments and that of LLM's judgment distribution. HelpSteer2, no-CoT. Bold denotes significant correlation ($\alpha = 0.01$).}
\label{tab:hs2-pointwise-agreement}
\end{table*}

\begin{table}[t]
\small
\centering
\begin{tabular}{ccccc}
\toprule
Model & Setting & Method & $W_1$ & $W_2$ \\
\midrule
\multirow{9}{*}{GPT-4o} & \multirow{3}{*}{point score} & mode & .229, .246 & .406, .419 \\
 &  & mean & .229, .247 & .388, \textbf{.349} \\
 &  & distr & \textbf{.219}, \underline{.222} & .395, .386 \\ \cline{2-5}
 & \multirow{3}{*}{pair score} & mode & .229, .230 & .419, .419 \\
 &  & mean & \underline{.218}, \textbf{.215} & .399, \textbf{.387} \\
 &  & distr & \underline{.220}, \textbf{.215} & .408, .401 \\ \cline{2-5}
 & \multirow{3}{*}{pair rank} & mode & .228, .226 & .420, .412 \\
 &  & mean & .221, .212 & .396, \textbf{.362} \\
 &  & distr & .215, \textbf{.203} & .405, .385 \\ \hline
\multirow{9}{*}{\shortstack{Llama\\-3.1-8B}} & \multirow{3}{*}{point score} & mode & .274, .267 & .438, .405 \\
 &  & mean & .277, .267 & .412, \textbf{.359} \\
 &  & distr & .261, \textbf{.246} & .425, .391 \\ \cline{2-5}
 & \multirow{3}{*}{pair score} & mode & .268, .276 & .460, .470 \\
 &  & mean & \underline{.241}, \underline{.244} & \underline{.404}, \textbf{.400} \\
 &  & distr & \textbf{.239}, \underline{.243} & .426, .433 \\ \cline{2-5}
 & \multirow{3}{*}{pair rank} & mode & \textbf{.296}, .336 & .490, .531 \\
 &  & mean & .356, .370 & \underline{.423}, \textbf{.420} \\
 &  & distr & .347, .356 & .540, .548 \\
\bottomrule
\end{tabular}
\caption{Pluralistic alignment error ($\downarrow$, Eq. \ref{eq:Wp}) from MT-Bench human pairwise preferences. Comma-separated values are with and without CoT. Text styling follows Table \ref{tab:pointwise-pairwise-short-main-results}. The method `distr' uses the predicted distribution, while the other methods place probability 1 on a measure of central tendency.}
\label{tab:mtbench-w}
\end{table}

\begin{table*}[t]
\small
\centering
\begin{tabular}{cccccccccccc}
\toprule
\multirow{2}{*}{Model} & \multirow{2}{*}{Method} & \multicolumn{2}{c}{Helpfulness} & \multicolumn{2}{c}{Correctness} & \multicolumn{2}{c}{Coherence} & \multicolumn{2}{c}{Complexity} & \multicolumn{2}{c}{Verbosity} \\\cmidrule(lr){3-4}\cmidrule(lr){5-6}\cmidrule(lr){7-8}\cmidrule(lr){9-10}\cmidrule(lr){11-12}
& & $W_1$ & $W_2$ & $W_1$ & $W_2$ & $W_1$ & $W_2$ & $W_1$ & $W_2$ & $W_1$ & $W_2$\\
\midrule
\multirow{3}{*}{GPT-4o} & mode & .218 & .311 & .219 & .332 & .149 & .252 & .211 & .273 & .186 & .257 \\
 & mean & .221 & .297 & .217 & .318 & .151 & .240 & .213 & .262 & .197 & \textbf{.244} \\
 & distr & \textbf{.188} & \textbf{.279} & \textbf{.194} & \textbf{.301} & \textbf{.134} & \textbf{.233} & \textbf{.199} & \textbf{.255} & \textbf{.179} & .249 \\ \hline
\multirow{3}{*}{\shortstack{Llama\\-3.1-8B}} & mode & .259 & .369 & .250 & .377 & .154 & .280 & .227 & .290 & .182 & .255 \\
 & mean & .255 & .339 & .249 & .347 & .158 & .253 & .224 & .274 & .174 & \textbf{.223} \\
 & distr & \textbf{.219} & \textbf{.328} & \textbf{.215} & \textbf{.334} & \textbf{.134} & \textbf{.250} & \textbf{.209} & \textbf{.270} & \textbf{.164} & .234 \\
\bottomrule
\end{tabular}
\caption{Pluralistic alignment error ($\downarrow$, Eq. \ref{eq:Wp}) from HelpSteer2 human pointwise scores. No-CoT. Text styling follows Table \ref{tab:pointwise-pairwise-short-main-results}. The method `distr' uses the predicted distribution, while the other methods place probability 1 on a measure of central tendency.}
\label{tab:hs2-pointwise-w}
\end{table*}

\subsection{Heterogenous Preferences}\label{app:hetero-preferences}

We investigate whether LLM judges can represent pluralistically aligned preferences (i.e. reflect diverse human opinions) \cite{sorensenposition, siththaranjan2023distributional, Kumar2024ComPOCP} through their judgment distribution, without explicit training or prompting.

\subsubsection{Multimodality}\label{app:multimodality}

We begin by quantifying the degree of multimodality in the judgment distributions. An implicit assumption behind the conventional method of using the mode judgment is that the judgment distribution is unimodal and thus the mode is a representative judgment. However, in cases where humans disagree, we would like LLM judges to reflect the heterogeneity in the human population with a multimodal distribution.

We quantify multimodality as the minimum amount of probability mass that must be added to make an unnormalized unimodal distribution, divided by the total mass of the unnormalized unimodal distribution to obtain a value in $[0, 1)$, where a distribution is unimodal if the probability mass function is non-decreasing and then non-increasing. For example, if the judgment distribution is $[0.5, 0.2, 0.3]$, the minimum additional mass is 0.1 to obtain the unimodal distribution $[0.5, 0.3, 0.3]$ with total mass 1.1, so we compute the multimodality as $0.1/1.1 \approx 0.091$.

Table \ref{tab:multimodality} presents the results. We find that more granularity leads to more multimodality (note that $K=2$ always has multimodality 0), and no-CoT is more multimodal than CoT. The case of extreme multimodality for pointwise scoring $K=99$ can be largely attributed to token bias \cite{Lovering2024AreLM, Shaikh2024CBEvalAF}. For example, GPT-4o $K=99$ CoT on MT-Bench assigns on average 0.036 probability to a single token that is a multiple of 5, but only 0.002 to a single token that differs by 1 from one of those multiples of 5.

\subsubsection{Annotator Disagreement}

We next examine whether human annotator disagreement is correlated with the uncertainty in the LLM's judgment distribution. On datasets with multiple human judgments per instance, we compute Spearman's $\rho$ between the standard deviation of the human judgments and that of the LLM's judgment distribution.

For MT-Bench, we take the 961 instances with multiple human judgments. Table \ref{tab:mtbench-agreement} reports weak correlation in all settings except no correlation in pairwise ranking with Llama-3.1-8B. Remarkably, \textit{pointwise} score distributions encode sufficient information to predict if humans will disagree on a \textit{pairwise} comparison of the texts.

The HelpSteer2 dataset \cite{Wang2024HelpSteer2OD} contains multiple human ratings on a 0-4 scale for five attributes for each (prompt, response) instance.
We use a random subset of size 1,000. We prompt with the provided annotation guidelines and have the model rate all attributes in a single run. Table \ref{tab:hs2-pointwise-agreement} reports weak correlation on helpfulness, correctness, and coherence but no correlation on complexity and verbosity. We suspected this to be due to that conditioning on the earlier attributes' scores may reduce uncertainty for the later attributes \cite{Stureborg2024LargeLM, Hashemi2024LLMRubricAM}, but we found that the average standard deviation is similar across attributes for both LLM and human judgments.

\subsubsection{Pluralistic Alignment}

We finally evaluate the alignment between predicted judgment distributions and human judgment distributions.
We quantify the distance between two distributions $\mu$ and $\nu$ with the Wasserstein $p$-distance for $p \in \{1, 2\}$:
\begin{align}
    W_p(\mu, \nu) = \inf_{\gamma \in \Gamma(\mu, \nu)} \left(\mathop{\mathbb{E}}_{(x, y) \sim \gamma} |x-y|^p\right)^{\frac1p},\label{eq:Wp}
\end{align}
where $\Gamma(\mu, \nu)$ is the set of couplings of $\mu$ and $\nu$.
A higher $p$ more heavily punishes large point distances $|x - y|$.
We scale the judgment spaces to $[0, 1]$ so that $W_p(\mu, \nu) \in [0, 1]$.

As baselines, we consider deterministic distributions that place probability 1 on a measure of central tendency.

Table \ref{tab:mtbench-w} shows that using a distributional prediction has little success in improving alignment with the MT-Bench human pairwise preferences, but Table \ref{tab:hs2-pointwise-w} shows success for HelpSteer2 human pointwise scores.

We also experimented with the HelpSteer2-Preference dataset, prompting with the provided annotation guidelines \cite{Wang2024HelpSteer2PreferenceCR}. However, we found severe position bias in our experiments with GPT-4o and Llama-3.1-8B (no-CoT). The analysis showed no correlation between predicted distribution variance and annotator disagreement, and poor pluralistic alignment compared to the deterministic baselines.

\subsection{Sensitivity to Score Granularity}\label{app:sensitivity-to-score-granularity}

\begin{table}[t]
\small
\centering
\begin{tabular}{cccc}
\toprule
Model & Method & Reward-Bench & MT-Bench \\
\midrule
\multirow{9}{*}{GPT-4o} & \cellcolor{gray!15}-- & \cellcolor{gray!15}.091, .105 & \cellcolor{gray!15}.093, .111 \\
 & \sc mode & .103, .150 & .128, .214 \\
 & \sc mean & \underline{.066}, .080 & \underline{.105}, .136 \\
 & \sc [mean] & .104, .115 & .144, .199 \\
 & \sc medi & .101, .113 & .137, .185 \\
 & \sc 1p & .096, .117 & .137, .196 \\
 & \sc ram & .074, .084 & .111, .138 \\
 & \sc qt & \textbf{.064}, .078 & \textbf{.104}, .133 \\
 & \sc ps & \textbf{.064}, .078 & \textbf{.104}, .137 \\ \hline
\multirow{9}{*}{\shortstack{Llama\\-3.1-8B}} & \cellcolor{gray!15}-- & \cellcolor{gray!15}.136, .063 & \cellcolor{gray!15}.117, .076 \\
 & \sc mode & .213, .201 & .223, .247 \\
 & \sc mean & .149, .042 & .131, \underline{.048} \\
 & \sc [mean] & .213, .139 & .219, .218 \\
 & \sc medi & .219, .160 & .224, .218 \\
 & \sc 1p & .223, .105 & .183, .133 \\
 & \sc ram & .168, \underline{.037} & .156, .068 \\
 & \sc qt & .151, \textbf{.034} & .130, \underline{.048} \\
 & \sc ps & .151, .037 & .129, \textbf{.046} \\
\bottomrule
\end{tabular}
\caption{Sensitivity to granularity ($\downarrow$) of the \hl{score distributions} (Eq. \ref{eq:Wp}) and of the pointwise methods computed on them (Eq. \ref{eq:pred-flip}). Comma-separated values are with and without CoT. Text styling follows Table \ref{tab:pointwise-pairwise-short-main-results}.}
\label{tab:pointwise-method-granularity-sensitivity-results}
\end{table}

Adopting the view that LLMs latently encode a continuous distribution but output a discretization of it \cite{Gillman2024FourierHH}, we analyze how faithfully functions of the (latent) continuous distribution can be approximated by those functions computed on the (observed) discretization. For practical interest, this manifests as robustness to the choice of $K$, with convergence in distribution to the continuous distribution as $K\to\infty$. Thus, independently of the ``principledness'' of certain functions of a ground-truth continuous distribution, it is appropriate to examine the effect of discretization on our ability to approximate them to begin with. Our theoretical result is stated in Proposition \ref{prop:approx-discrete} (see Appendix \ref{app:approx} for full statement, proof, and discussion).

\begin{repproposition}{prop:approx-discrete}
    Among the discrete methods in Table \ref{tab:pointwise-methods}, {\sc mode} computed on continuous distributions may fail to be approximated by the same function computed on their discretizations, even under regularity conditions. Meanwhile, {\sc[mean]}, {\sc medi}, and {\sc 1p} admit an approximation error bound.
\end{repproposition}

We empirically assess the robustness to $K$ of the score distributions produced by the LLM judge as well as the functions computed on them. The former is not addressed by Proposition \ref{prop:approx-discrete}, which assumes the score distributions to be errorless discretizations and thus consistent across granularities.

\subsubsection{Sensitivity of Score Distributions}
For an evaluated text, let $\mu^K$ denote the score distribution with granularity $K$, with the score space scaled to $[0, 1]$. We coarsify $\mu^{99}$ into $\hat{\mu}^{99}$ by binning into 9 blocks of 11 scores. We then quantify sensitivity as the Wasserstein 1-distance $W_1(\mu^9, \hat{\mu}^{99}) \in [0, 1]$ (Eq. \ref{eq:Wp}) averaged over the pointwise instances in the dataset.

\subsubsection{Sensitivity of Pointwise Methods}
For a dataset $\mathcal{D}$ of paired responses, we denote $\mathbf{a}^K$ as the $|\mathcal{D}|$-length vector containing the value of a method computed on each pair using granularity $K$. We then quantify sensitivity as the normalized flip rate
\begin{align}
    {\rm FR} \coloneq \frac{\|{\rm sgn}(\mathbf{a}^9) - {\rm sgn}(\mathbf{a}^{99})\|_1}{\|{\rm sgn}(\mathbf{a}^9)\|_1 + \|{\rm sgn}(\mathbf{a}^{99})\|_1} \in [0, 1].\label{eq:pred-flip}
\end{align}

\subsubsection{Results}

Table \ref{tab:pointwise-method-granularity-sensitivity-results} presents the results on sensitivity to granularity.
The discrete metrics are more sensitive than the continuous metrics.
Furthermore, consistent with Proposition \ref{prop:approx-discrete}, we find that the mode is the most sensitive among the discrete methods, particularly with no-CoT.

The effect of CoT differs between the models: GPT-4o is less sensitive with CoT, and Llama-3.1-8B is less sensitive with no-CoT. Similar to \citet{Lee2024EvaluatingTC}, it would appear that although GPT-4o is a more capable judge than Llama-3.1-8B, it is not as robust to granularity (in each model's CoT/no-CoT of choice). However, this is partially because a limitation with setting $K$ as large as 99 for GPT-4o is that no-CoT distributions tend to have high spread (Table \ref{tab:std}), resulting in nontrivial probability mass falling outside of the top 20 tokens provided by the OpenAI API. Concretely, the average total mass on the top score tokens is 0.88/0.90 on RewardBench/MT-Bench for no-CoT, but over 0.99 for CoT.
\subsection{Position Bias}\label{app:position-bias}

\begin{table}[t]
\small
\centering
\begin{adjustbox}{center=\columnwidth}
\begin{tabular}{ccccc}
\toprule
Model & Setting & $K$ & MAE & MSE \\
\midrule
\multirow{5}{*}{GPT-4o} & score & 9 & .090, \textbf{.076} & .057, \textbf{.031} \\
 & score & 99 & .094, .095 & .049, \underline{.032} \\
 & rank & 2 & .086, .087 & .083, .037 \\
 & rank & 3 & .085, .089 & .078, .035 \\
 & rank & 5 & .141, .182 & .079, .053 \\ \hline
\multirow{5}{*}{\shortstack{Llama\\-3.1-8B}} & score & 9 & .199, \underline{.163} & .125, .066 \\
 & score & 99 & .188, \textbf{.160} & .114, \textbf{.060} \\
 & rank & 2 & .357, .329 & .193, .154 \\
 & rank & 3 & .683, .518 & .547, .340 \\
 & rank & 5 & .506, .342 & .334, .164 \\
\bottomrule
\end{tabular}
\end{adjustbox}
\caption{Pairwise position bias ($\downarrow$, see Appendix \ref{app:position-bias}) on RewardBench (see Table \ref{tab:pairwise-position-mtbench-results} for MT-Bench). Comma-separated values are with and without CoT. Text styling follows Table \ref{tab:pointwise-pairwise-short-main-results}. We find that no-CoT always maintains or improves MSE, even when it hurts MAE.}
\label{tab:pairwise-position-rewardbench-results}
\end{table}

\begin{table}[t]
\small
\centering
\begin{adjustbox}{center=\columnwidth}
\begin{tabular}{ccccc}
\toprule
Model & Setting & $K$ & MAE & MSE \\
\midrule
\multirow{5}{*}{GPT-4o} & score & 9 & .108, \textbf{.091} & .075, \textbf{.038} \\
 & score & 99 & .111, .108 & .066, \textbf{.039} \\
 & rank & 2 & .108, .132 & .100, .056 \\
 & rank & 3 & .108, .134 & .093, .051 \\
 & rank & 5 & .187, .172 & .120, .047 \\ \hline
\multirow{5}{*}{\shortstack{Llama\\-3.1-8B}} & score & 9 & .211, .148 & .145, .056 \\
 & score & 99 & .193, \textbf{.141} & .129, \textbf{.049} \\
 & rank & 2 & .312, .355 & .174, .172 \\
 & rank & 3 & .618, .532 & .466, .337 \\
 & rank & 5 & .458, .298 & .293, .129 \\
\bottomrule
\end{tabular}
\end{adjustbox}
\caption{Pairwise position bias ($\downarrow$) on MT-Bench, mirroring Table \ref{tab:pairwise-position-rewardbench-results}.}
\label{tab:pairwise-position-mtbench-results}
\end{table}

\begin{table}[t]
\small
\centering
\setlength{\tabcolsep}{4pt}
\begin{tabular}{cccc}
\toprule
Space & Nectar & RM-Bench & MT-Bench \\
\midrule
interm & \textbf{.086} & \textbf{.079} & \textbf{.033} \\
list & .092 & .100 & \underline{.041} \\
direct list & .118 & .105 & .056 \\
\bottomrule
\end{tabular}
\caption[]{Listwise position bias ($\downarrow$) with GPT-4o. We report the absolute value\footnotemark{} of Spearman's $\rho$ between the difference in the presented positions of two responses and the judgment. Text styling follows Table \ref{tab:pointwise-pairwise-short-main-results}.}
\label{tab:listwise-position-results}
\end{table}

We compare the degree of position bias (i.e. the LLM judge's sensitivity to the order in which the evaluated texts are presented \cite{zheng2023judging}) between various settings.

\paragraph{Evaluation Metrics}

For the pairwise setting (scoring or ranking), we measure mean absolute error (MAE) and mean squared error (MSE) between the two judgments from the two orders, using pre-aggregation mean.
Compared to MAE, MSE punishes a few large errors more than many small errors.

For the listwise setting, we measure Spearman's $\rho$ between the difference in the presented positions of two responses and the judgment.

\footnotetext{In every judgment space, GPT-4o tends to favor responses that are presented earlier.}

\paragraph{Results}

Tables \ref{tab:pairwise-position-rewardbench-results} and \ref{tab:pairwise-position-mtbench-results} report position bias in the pairwise settings. We find that no-CoT always improves MSE, even when it hurts MAE, showing that no-CoT reduces cases of extreme position bias.

Table \ref{tab:listwise-position-results} reports listwise position bias. We find that {\sc direct list} exhibits the most position bias, consistent with \citet{zhu2024starling}, despite achieving the highest accuracy (Table \ref{tab:listwise-results}). On the other hand, {\sc interm} has the least position bias. As the intermediate pairwise preferences can be likened to CoT, this suggests that intermediate reasoning can mitigate bias in challenging judgment settings. However, since an ideal judge should be able to simultaneously maximize accuracy and minimize bias, we believe current methods have ample room for improvement.

\subsection{Transitivity}\label{app:transitivity}

\begin{table}[t]
\small
\centering
\begin{tabular}{cccc}
\toprule
Model & Setting & Method & MT-Bench \\
\midrule
\multirow{4}{*}{GPT-4o} & point score & \sc qt & .000, .000 \\
 & point score & \sc ps & .006, .002 \\
 & pair rank & \sc mode-agg & .026, .022 \\
 & pair rank & \sc agg-mean & .007, .003 \\ \hline
\multirow{4}{*}{\shortstack{Llama\\-3.1-8B}} & point score & \sc qt & .000, .000 \\
 & point score & \sc ps & .001, .000 \\
 & pair rank & \sc mode-agg & .234, .218 \\
 & pair rank & \sc agg-mean & .040, .023 \\
\bottomrule
\end{tabular}
\caption{A study on transitivity. In each cell, we report the proportion of triplets that exhibit intransitivity, with and without CoT. (Pointwise scoring uses $K=9$; pairwise ranking uses Likert-2.) In addition, our Nectar silver labels (GPT-4o, Likert-5, no-CoT, mean) have an intransitivity rate of 0.020.}
\label{tab:intrans}
\end{table}

We say a comparison method $a(\cdot, \cdot) \in [-1, 1]$ is transitive if $a(A_1, A_2) > 0$ and $a(A_2, A_3) \ge 0$ imply $a(A_1, A_3) > 0$ for all triplets of texts $(A_1, A_2, A_3)$. For example, a score distribution comparison function that reduces to the comparison of two real numbers derived from the two score distributions independently (e.g. mode or mean) is transitive. On the other hand, {\sc qt}, {\sc ps}, and the pairwise ranking methods are intransitive.

Human preferences have been shown to exhibit intransitivity \cite{klimenko2015intransitivity}, motivating the question of whether LLM judges do so too and how this depends on the method used. Several prior works have proposed methods incorporating awareness of the intransitivity in LLM or human preferences \cite{liu2024aligning, ethayarajh2024kto, Zhang2024GeneralPM, Ye2024OnlineIR, Hu2024LanguageMP, Zhang2024ContraSolverSO, Liu2024AligningWL}. We adopt the view in \citet{liu2024aligning} that transitivity is generally desirable and indicative of a more capable judge, especially in the absence of a curated dataset of intransitive human preferences. Nevertheless, we remark that the ability to model intransitivity is essential to preference modeling in its full generality \cite{ethayarajh2024kto, Zhang2024GeneralPM, Ye2024OnlineIR}, which, among pointwise methods, is achieved by {\sc qt} and {\sc ps} but not by mode and mean used in prior work.

Table \ref{tab:intrans} presents the intransitivity rates of different methods. Despite the capacity of {\sc qt} and {\sc ps} to model intransitive preferences \cite{Savage1994ThePO, Finkelstein2006NontransitiveDW, Conrey2013IntransitiveD}, we find that they exhibit negligible intransitivity compared to the pairwise ranking methods. Similar to \citet{liu2024aligning}, we observe that a stronger judge (GPT-4o) exhibits less intransitivity than a weaker judge (Llama-3.1-8B). Pre-aggregation mean exhibits less intransitivity than post-aggregation mode. Notably, for pairwise ranking, we observe more intransitivity with CoT than without CoT, even though CoT achieves higher accuracy (Table \ref{tab:pointwise-pairwise-short-main-results}).

\section{Derivations}\label{app:derive}

\subsection{Approximability of Discrete Pointwise Functions Under Discretization}\label{app:approx}

\begin{proposition}\label{prop:approx-discrete}
    We analyze the discrete methods in Table \ref{tab:pointwise-methods}. Specifically, we examine the score function $r$ rather than ${\rm sgn}(r_1-r_2)$.

    Let $X$ be a random variable with support $S \subset [\frac12, K + \frac12)$ for an integer $K$. Define its discretization $\hat{X}$ by $P(\hat{X} = \hat{x}) \coloneq P([X] = \hat{x})$ for $\hat{x} \in \hat{S} \coloneq \{1, \ldots, K\}$, where $[\cdot]$ denotes rounding to the nearest integer.
    \begin{enumerate}
        \item {\sc mode} may fail to be approximated: Suppose $X$ has a density $f_X$ that is $L$-Lipschitz with $L \le 1$ and achieves its supremum at $x^* \in \arg\max_{x \in S} f_X(x)$. Let $\hat{x}^* \in \arg\max_{\hat{x} \in \hat{S}} P(\hat{X} = \hat{x})$.
        Suppose some $\hat{x} \in \hat{S}$, with arbitrarily large $|\hat{x} - \hat{x}^*| > 1$, satisfies $P(\hat{X} = \hat{x}^*) \ge P(\hat{X} = \hat{x}) + \frac{L}{4}$. The above is consistent with $[x^*] = \hat{x}$.



        \item {\sc[mean]} can be approximated: $|[\mathbb{E}X] - [\mathbb{E}\hat{X}]| \le 1$.

        \item {\sc medi} and {\sc 1p} can be approximated: For $p \in (0, 1)$, $|Q_X(p) - Q_{\hat{X}}(p)| \le \frac12$.

    \end{enumerate}
\end{proposition}

\begin{proof}\hfill
    \begin{enumerate}
        \item\label{prop:approx-mode} We present a construction.

        If $L = 0$, the claim is immediate; assume not. Define $d \coloneq \frac{L}{4}(\sqrt{1+8/L}-2) \ge \frac{L}{4}$. Let $f_X(x) = (d - \frac{L}{4}) + L(x - \hat{x} + \frac12)$ for $x \in [\hat{x} - \frac12, \hat{x})$, and $f_X(x) = (d - \frac{L}{4}) + L(\hat{x} - x + \frac12)$ for $x \in [\hat{x}, \hat{x} + \frac12)$, and $f_X(x) = d + \frac{L}{4}$ for $[x] = \hat{x}^*$.
        
        Around the regions $[\hat{x}-\frac12, \hat{x}+\frac12), [\hat{x}^*-\frac12, \hat{x}^*+\frac12)$, we let $f_X$ decrease to 0 with slope $\pm L$, or until reaching the domain boundary or each other. Continuity is maintained at the junction because, supposing $\hat{x} < \hat{x}^*$ without loss of generality, the nearest endpoints $\hat{x} + \frac12, \hat{x}^* - \frac12$ satisfy $|(\hat{x} + \frac12) - (\hat{x}^* - \frac12)| \ge 1$ and $|f_X(\hat{x} + \frac12) - f_X(\hat{x}^* - \frac12)| = \frac{L}{2}$.

        We verify that $P(\hat{X} = \hat{x}^*) = d + \frac{L}{4} = P(\hat{X} = \hat{x}) + \frac{L}{4}$ and $\hat{x} \in \{\hat{x}\} \cup [\hat{x}^* - \frac12, \hat{x}^* + \frac12) = \arg\max_{x \in S} f_X(x)$.

        It remains to check that we have a valid distribution. The total $\int f_X$ is bounded by the case if $f_X$ is allowed to reach 0 everywhere possible above:
        \begin{align*}
            \int f_X &\le P(\hat{X} = \hat{x}) + P(\hat{X} = \hat{x}^*)\\
            &\quad+ \frac{1}{L}\left(d-\frac{L}{4}\right)^2 + \frac{1}{L}\left(d+\frac{L}{4}\right)^2\\
            &= 1 - \frac{L}{4} < 1,
        \end{align*}
        so $f_X$ can be made a valid density by adding an appropriately scaled uniform density, not affecting the desired properties.



        \item\label{prop:approx-mean} Denote the measures of $X, \hat{X}$ as $\mu_X, \mu_{\hat{X}}$. The definition of $X$ and $\hat{X}$ is equivalent to the existence of a coupling $\gamma \in \Gamma(\mu_X, \mu_{\hat{X}})$ with samples defined by $(x, \hat{x}) \sim \gamma$ for $x \sim \mu_X$ and $\hat{x} = [x]$.
        \begin{gather*}
            |\mathbb{E}X - \mathbb{E}\hat{X}| = \left|\int (x - \hat{x}) \d \gamma(x, \hat{x})\right|\\
            \le \int |x - \hat{x}| \d \gamma(x, \hat{x}) \le \int \frac12 \d \gamma(x, \hat{x}) = \frac12
        \end{gather*}
        Thus, $|[\mathbb{E}X] - [\mathbb{E}\hat{X}]| \le 1$.

        \item\label{prop:approx-quantile} Let $q \coloneq Q_X(p)$.
        \begin{gather*}
            P(\hat{X} < [q]-\frac12) = P(X < [q]-\frac12) < p\\
            \le P(X < [q]+\frac12) = P(\hat{X} < [q]+\frac12),
        \end{gather*}
        implying $Q_{\hat{X}}(p) = [q]$ where $|q - [q]| \le \frac12$.

    \end{enumerate}
\end{proof}
\begin{remark}
    The suppositions in (\ref{prop:approx-mode}) are to impose regularity and show even then approximation may not hold. For an example of their omission, without requiring absolutely continuous $X$, it could place atoms at arbitrary $x$, preventing any margin $P(\hat{X} = \hat{x}^*) - P(\hat{X} = \hat{x})$ less than 1 from producing an error bound. The crucial case that causes the mode to be unstable to approximate is the case of multimodality.


    In (\ref{prop:approx-quantile}), it is crucial that we assumed no discretization error, i.e. $|P(\hat{X} = \hat{x}) - P([X] = \hat{x})| = 0$. With any discretization error, we would have no bound on approximation error.

\end{remark}

\section{Licensing}

Our usage of the artifacts below complies with their licenses.

\paragraph{Model Licensing}
GPT-4o\footnote{\url{https://platform.openai.com/docs/models\#gpt-4o}} has a proprietary license. Llama-3.1-8B\footnote{\url{https://huggingface.co/meta-llama/Llama-3.1-8B-Instruct}} is licensed under the Llama 3.1 Community License Agreement. Mistral-7B\footnote{\url{https://huggingface.co/mistralai/Mistral-7B-Instruct-v0.3}} and Prometheus-2-7B\footnote{\url{https://huggingface.co/prometheus-eval/prometheus-7b-v2.0}} are licensed under the Apache License 2.0.

\paragraph{Dataset Licensing}
The datasets contain English language data.
RewardBench\footnote{\url{https://huggingface.co/datasets/allenai/reward-bench}} and RM-Bench\footnote{\url{https://huggingface.co/datasets/THU-KEG/RM-Bench}} are licensed under the ODC-By license. MT-Bench\footnote{\url{https://huggingface.co/datasets/lmsys/mt_bench_human_judgments}} and HelpSteer2\footnote{\url{https://huggingface.co/datasets/nvidia/HelpSteer2/tree/main/disagreements}} are licensed under the CC BY 4.0 license. Nectar\footnote{\url{https://huggingface.co/datasets/berkeley-nest/Nectar}} is licensed under the Apache License 2.0.

\section{Ethical Considerations}

LLMs can exhibit unwanted biases. Relying on their judgments for downstream applications can propagate these biases. Nevertheless, our findings in this paper promote practices for improving alignment with human preferences.

\end{document}